\documentclass{article}
\usepackage[arxiv]{icml2025}

\usepackage{amsmath,amsfonts,amsthm,amssymb,bm}
\usepackage{caption}
\usepackage{subcaption}
\usepackage{multicol}

\def\eqref#1{equation~\ref{#1}}

\def\1{\bm{1}}

\DeclareMathAlphabet{\mathsfit}{\encodingdefault}{\sfdefault}{m}{sl}
\SetMathAlphabet{\mathsfit}{bold}{\encodingdefault}{\sfdefault}{bx}{n}

\usepackage{paracol}
\newcommand{\pidraft}{\pi_{\mathrm{draft}}}
\newcommand{\pitarget}{\pi_{\mathrm{target}}}
\newcommand{\Cdt}{C_{\mathtt{seq}}}
\newcommand{\Cdtstep}{C_{\mathtt{block}}}
\newcommand{\kl}{D_{\text{KL}}}
\newcommand{\rprm}{r_{\mathtt{PRM}}}
\newcommand{\rprmtarget}{r_{\mathtt{PRM}}^{\text{target}}}

\newcommand{\subsample}{\textsc{\texttt{SubSample}}\xspace}

\usepackage{comment}
\newcommand{\currtrace}{\textsf{tr}}

\newcommand{\Ydraft}{\mathbf{Y}_{\text{draft}}}
\newcommand{\Yaccept}{\mathbf{Y}_{\text{accept}}}
\newcommand{\Ytarget}{\mathbf{Y}_{\text{target}}}
\newcommand{\pibase}{\pi_{\mathrm{base}}}
\newcommand{\pigen}{\pi_{\mathrm{gen}}}
\newcommand{\pitwostage}{\pi_{\texttt{two-stage},\beta_1,\beta}}
\newcommand{\pipartapx}[1]{\pi_{\texttt{SMC},#1}}
\newcommand{\barpipartapx}[1]{\overline{\pi}_{\texttt{SMC},#1}}
\newcommand{\Poi}{\textsf{Poi}}
\newcommand{\bbE}{\mathbb{E}}
\newcommand{\prm}{\texttt{PRM}}

\usepackage{soul,ulem}
\newcommand\redsout{\bgroup\markoverwith{{\rule[0.5ex]{3pt}{0.75pt}}}\ULon}

\usepackage{float}
\newfloat{myfloat}{htbp}{lop}
\floatname{myfloat}{Experiment}

\usepackage{enumitem}
\newlist{myenum}{enumerate}{1} 

\setlist[myenum]{label=\arabic*., leftmargin=*, itemsep=0em}
\usepackage{scrextend}
\makeatletter
\newcommand{\LineLabel}[1]{\edef\@currentlabel{\theALC@line}\label{#1}}
\makeatother

\usepackage{ragged2e}

\usepackage{hyperref}
\usepackage{url}
\usepackage{booktabs}
\usepackage{siunitx}
\usepackage{microtype}
\usepackage{graphicx}
\usepackage{booktabs} 
\usepackage{multirow}
\usepackage{hyperref}
\usepackage{makecell}
\usepackage{amsmath}
\usepackage{amssymb}
\usepackage{mathtools}
\usepackage{amsthm}
\usepackage{arydshln}
\usepackage{cleveref}

\theoremstyle{plain}
\newtheorem{theorem}{Theorem}

\newtheorem{lemma}{Lemma}
\newtheorem{corollary}{Corollary}
\theoremstyle{definition}
\newtheorem{definition}{Definition}
\newtheorem{assumption}{Assumption}
\newtheorem{observation}{Observation}
\theoremstyle{remark}
\newtheorem{remark}{Remark}
\usepackage{url}
\usepackage{tikz}
\usepackage{pgfplots}
\usepackage{physics}
\usepackage{comment}
\usepackage{xspace}
\usepackage{graphicx}
\usepackage{algorithm}
\usepackage{algorithmic}

\usepackage[textsize=tiny]{todonotes}

\newcommand{\specalign}{{\texttt{SPECS}}\xspace}

\def\withnotes{1}
\ifnum\withnotes=1
\newcommand{\nived}[1]{{\color{red} \textbf{Nived:} #1}}
\newcommand{\ahmad}[1]{{\color{teal} \textbf{Ahmad:} #1}}
\newcommand{\mert}[1]{{\color{brown} \textbf{Mert:} #1}}
\newcommand{\rishabh}[1]{{\color{red} \textbf{Rishabh:} #1}}
\newcommand{\lily}[1]{{\color{brown} \textbf{Lily:} #1}}
\newcommand{\ziteng}[1]{{\color{orange} \textbf{Ziteng:} #1}}
\newcommand{\ion}[1]{{\color{blue} \textbf{Ion:} #1}}
\newcommand{\kannan}[1]{{\color{blue} \textbf{Kannan:} #1}}
\else
\newcommand{\nived}[1]{{}}
\newcommand{\ahmad}[1]{{}}
\newcommand{\mert}[1]{{}}
\newcommand{\rishabh}[1]{{}}
\newcommand{\lily}[1]{}
\newcommand{\ziteng}[1]{{}}
\newcommand{\ion}[1]{{}}
\newcommand{\kannan}[1]{{}}
\fi
\usepackage[table]{xcolor} 
\usepackage{xcolor}
\usepackage{booktabs}
\usepackage{colortbl}
\usepackage{tabulary}
\usepackage{etoolbox}
\definecolor{darkgreen}{rgb}{0.0, 0.5, 0.0}
\definecolor{darkred}{rgb}{0.6, 0.0, 0.0}
\newcommand{\Smodel}{\color{gray}BS (\texttt{\textcolor{darkgreen}{S})}}
\newcommand{\Bmodel}{BS (\texttt{\textcolor{darkred}{B})}}

\newcommand{\specaligntable}{\texttt{$\specalign$}}
\newcommand{\graycell}[1]{\cellcolor{teal!15}#1}

\newcommand{\commentsymbol}{//}
\makeatletter
\newcommand{\LineComment}[1]{\hfill\commentsymbol{} #1}
\makeatother

\usepackage{titletoc}
\usepackage{etoolbox}

\pretocmd{\appendix}{%
  \pretocmd{\section}{\addcontentsline{apx}{section}{\thesection\quad #1}}{}{}%
}{}{}

\icmltitlerunning{$\specalign$: Faster Test-Time Scaling}
\begin{document}

\twocolumn[
  \icmltitle{$\specalign$: Faster Test-Time Scaling through \\ Speculative Drafts and Dynamic Switching}

\icmlsetsymbol{equal}{*}

\begin{icmlauthorlist}
\icmlauthor{Mert Cemri}{equal,yyy}
\icmlauthor{Nived Rajaraman}{equal,yyy}
\icmlauthor{Rishabh Tiwari}{equal,yyy}
\icmlauthor{Xiaoxuan Liu}{yyy}
\icmlauthor{Kurt Keutzer}{yyy}\\
\icmlauthor{Ion Stoica}{yyy}
\icmlauthor{Kannan Ramchandran}{yyy}
\icmlauthor{Ahmad Beirami}{sch}
\icmlauthor{Ziteng Sun}{comp}

\end{icmlauthorlist}

\icmlaffiliation{yyy}{UC Berkeley}
\icmlaffiliation{comp}{Google Research}
\icmlaffiliation{sch}{Google DeepMind}
\icmlcorrespondingauthor{\\Mert Cemri}{cemri@berkeley.edu}
\icmlcorrespondingauthor{\\Nived Rajaraman}{nived.rajaraman@berkeley.edu}
\icmlcorrespondingauthor{\\Rishabh Tiwari}{rishabhtiwari@berkeley.edu}
\icmlcorrespondingauthor{\\Ahmad Beirami}{ahmad.beirami@gmail.com}
\icmlcorrespondingauthor{\\Ziteng Sun}{zitengsun@google.com}
\icmlkeywords{Machine Learning, ICML}
  \vskip 0.3in
] 


\printAffiliationsAndNotice{\icmlEqualContribution}

\begin{abstract}
Scaling test-time compute has driven the recent advances in the reasoning capabilities of large language models (LLMs). 
However, increased compute often comes at the expense of higher user-facing latency, directly impacting user experience. Current test-time scaling methods primarily optimize for accuracy based on total compute resources (FLOPs), often overlooking latency constraints. To address this gap, we propose \specalign, a latency-aware test-time scaling method. \specalign builds upon beam search, which generates multiple reasoning traces for each step with a reasoning model, and selects one to continue from based on the scores from a dedicated reward model. Inspired by speculative decoding, \specalign~uses a smaller, faster model to generate candidate traces efficiently, and evaluates these candidates with both the reasoning model and the reward model. We design novel strategies to select candidate drafts using these model evaluations, including reward-guided soft verification, and a dynamic switching mechanism to defer to the larger model on harder steps. Empirical results on MATH500, AMC23, OlympiadBench and GPQA datasets show that \specalign~matches or surpasses the accuracy of beam search while reducing latency by up to $\sim$18\%. Our theoretical analysis shows that our algorithm converges to the solution of a KL-regularized reinforcement learning objective as the beam width grows \footnote{The code is available at \hyperlink{https://github.com/mert-cemri/SPECS}{https://github.com/mert-cemri/SPECS}}. 
\end{abstract}

\section{Introduction}
Modern LLMs excel at multi-step reasoning, and scaling test-time compute has played a major role in achieving these reasoning capabilities by letting these models tackle harder problems with  extra “thinking” compute resources~\citep{cobbe2021trainingverifierssolvemath,wei2022chain,beirami2024theoretical,brown2024largelanguagemonkeysscaling,beeching2024scalingtesttimecompute,qiu2024treebon,o1,deepseekai2025deepseekr1incentivizing}. 

To date, test-time scaling methods have primarily optimized performance based on total compute, demonstrating improved downstream task performance with increased FLOPs or generated tokens~\citep{o1,snell2025scaling}. However, user experience often depends more directly on serving \textit{latency}, especially in low-throughput scenarios like personalized interactions or applications serving few users \citep{patel2024splitwise, agrawal2024taming}. Autoregressive generation from LLMs often operates in the regime where the latency is limited by memory loading rather than total FLOPs~\citep{tiwari2025quantspec,yuan2024llminferenceunveiledsurvey}. Consequently, we study the critical research question:

\begin{center}
\textit{Can we design test-time scaling methods which optimize for inference latency while maintaining accuracy?}
\end{center}
The problem of improving the latency of autoregressive generation has received significant attention, and resulted in approaches such as speculative decoding~\citep{leviathan2022fast,chen2023acceleratinglargelanguagemodel}. Speculative decoding uses a smaller, faster draft model to propose candidate tokens, which are validated in parallel by a larger target model, reducing the total memory loading time and hence the overall decoding latency. 

\begin{figure*}
    \centering
\includegraphics[trim=225 300 225 150, clip, width=0.9\linewidth]
    {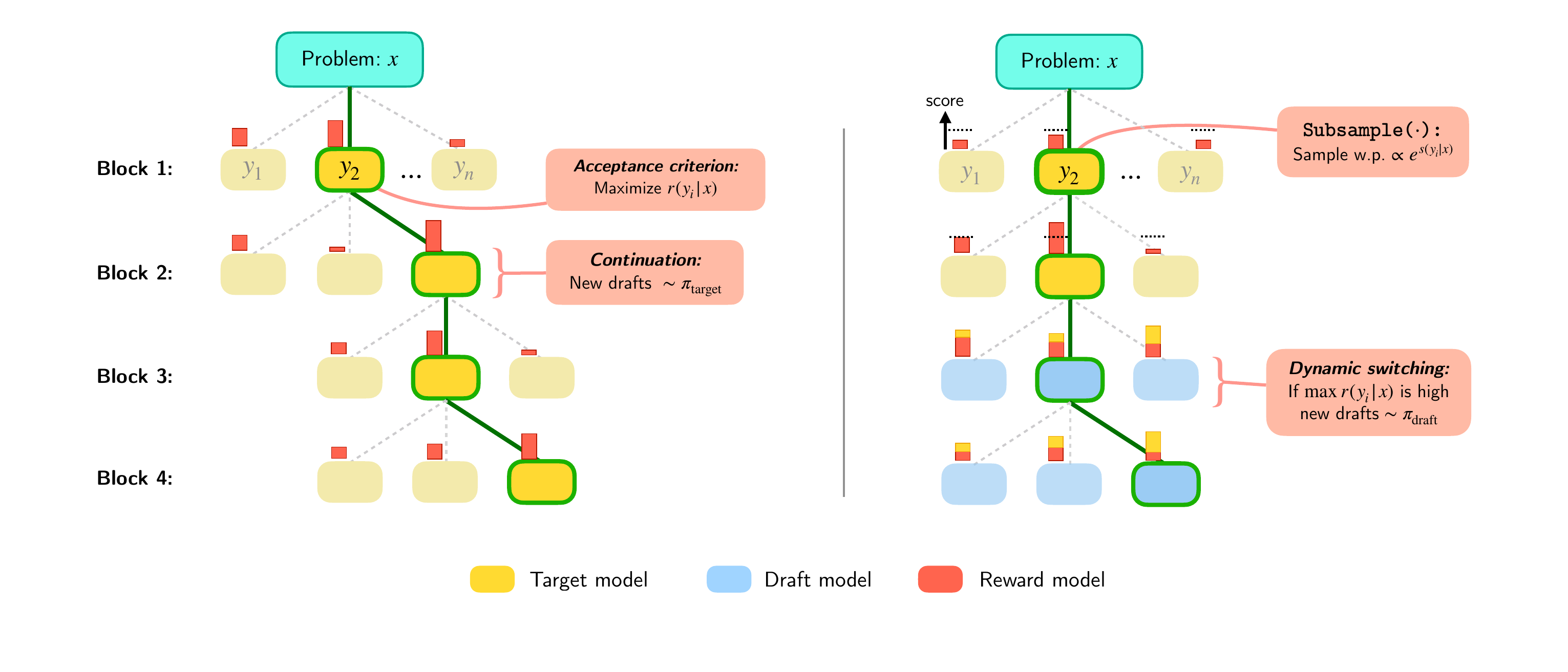}
    \caption{Visualization of Beam Search vs \specalign. In beam search the trajectories are generated by the target model ($p$) and scored using a PRM ($r$). In contrast, in \specalign the beams are dynamically switched to generation from the draft model, and scored by a combination of target and PRM model, resulting in better latency-performance tradeoff. Draft proposal and selection are further controlled by the \subsample subroutine (see \cref{sec:algorithm} for details).}
    \label{fig:special}
    \vspace{-1.2em}
\end{figure*}

In this work, we study how speculative drafts can be used to develop new test-time scaling methods which optimize for latency. Our investigation builds on the family of beam search or tree search~\citep{mudgal2023controlled, beeching2024scalingtesttimecompute,sun2024fastbestofndecodingspeculative,qiu2024treebon,snell2025scaling} algorithms for test-time scaling. In each iteration multiple candidate reasoning steps (i.e., beams) are generated in parallel from a target model, and the most promising ones are kept to continue generation from, based on a reward signal (i.e., a process reward model (PRM) \citep{zhang2025lessons} or self rewards \citep{yuan2024self}). The multi-step structured exploration makes them particularly effective at balancing accuracy and computational effort. At the same time, it also offers the flexibility of dynamically switching reasoning models for easier/harder steps.

We propose a new algorithm, \specalign~(SPECulative drafting for faster test-time Scaling), which relies on a speculative drafting step for improving latency. A speculative drafting step uses a fast draft model to propose candidate reasoning steps, and jointly uses a larger, more capable target model and a reasoning-specific PRM to select from the candidates. To make sure that our algorithm incurs minimal decrease in accuracy, we design novel dynamic switching scheme to identify reasoning steps where speculative drafting would not hurt latency, as well as a theoretically-grounded score metric integrating the evaluations of the large model and the reward model for better draft selection.

Our contributions can be summarized as follows:
\begin{myenum}
    \item We propose a novel test-time scaling algorithm, \specalign (\Cref{alg:special}), which utilizes speculative drafts to reach favorable latency-accuracy tradeoff. Our algorithm starts by generating from the target model and dynamically switches to speculative drafting for high reward traces to reduce latency while maintaining accuracy. This builds on the novel insight that the accuracy gap between  a small model and a large model is small for these traces, which could be of independent interest. To the best of our knowledge, other works building on speculative drafting all start with a draft model and switch to target model when the reward signal is low, resulting in lower accuracy or wasted latency for generating speculative drafts.
     \item We design a novel strategy to integrate scores from target models and PRMs for better draft selection.      Theoretically, we show that the approach converges gracefully to the optimal solution of a KL-regularized reward maximization objective as beam-width (i.e., parallel compute) increases, demonstrating the theoretical grounding of the proposed approach. We present a novel analysis of the soft best-of-$N$ algorithm, and extends this to speculative drafting and multi-step reasoning.
    \item We evaluate on MATH-500, AMC23, OlympiadBench and GPQA datasets, demonstrating that \specalign achieves up to 18\% reduction in latency while achieving on-par or higher accuracy compared to beam search and other baselines, using (Qwen-1.5B-Instruct \& Qwen-7B-Instruct) and (Llama3.2-1B-Instruct \& Llama3.1-8B-Instruct) as draft-target model pairs with Skywork-o1-Open-PRM-Qwen-2.5-7B as the process reward model.
\end{myenum}

We conduct ablation studies to show that both design components contribute significantly in improving the latency-accuracy curve over baselines. We refer readers to \Cref{sec:related} for a detailed comparison to related work due to space constraints.
\vspace{-0.5em}
\section{Speculative drafting with reward-guided soft verification} \label{sec:algorithm}
\vspace{-0.5em}


\begin{figure}[h!]
  \centering
  \begin{subfigure}[b]{0.2\textwidth}
    \centering
    \includegraphics[width=\linewidth]{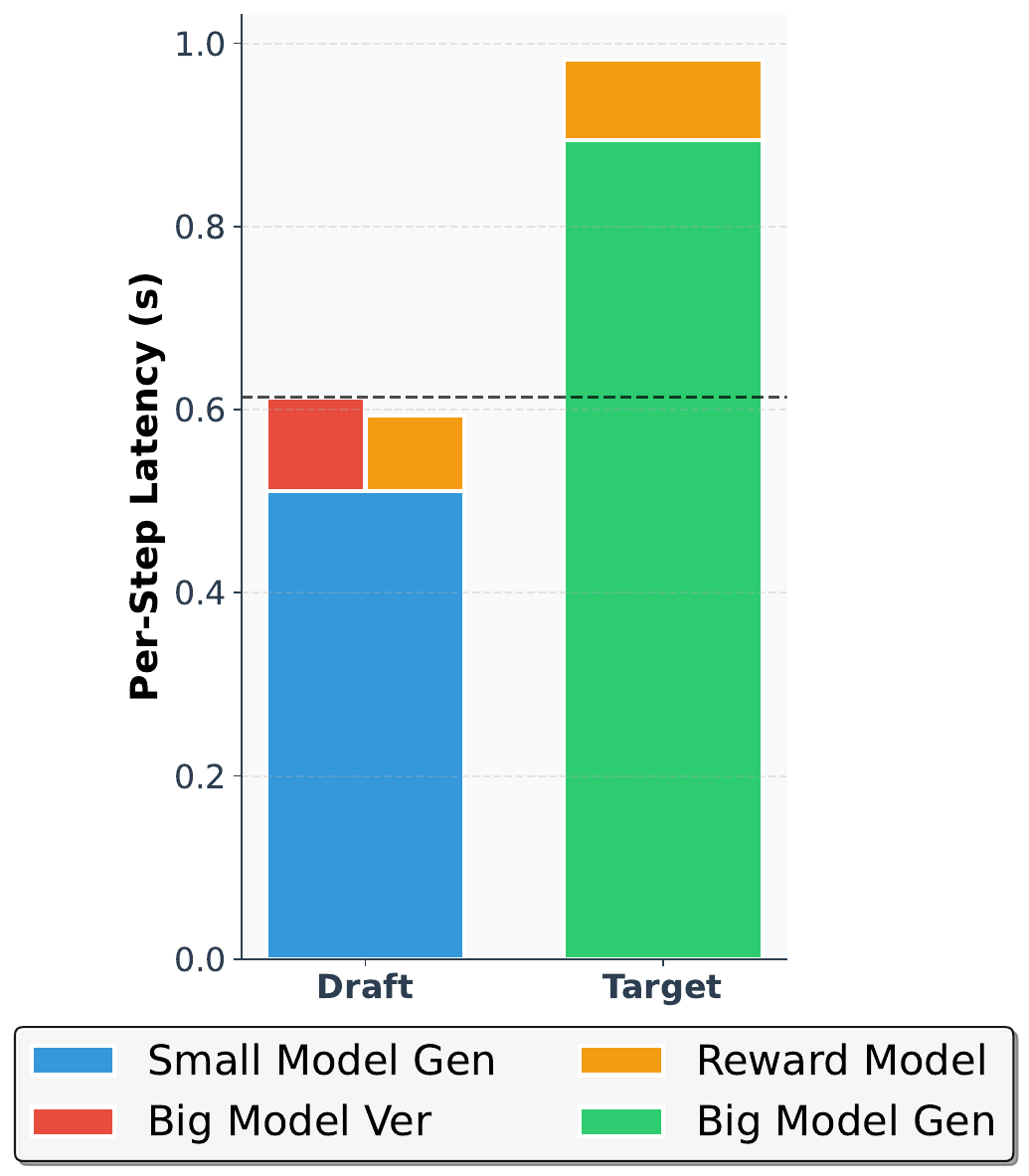}
    \caption{}
    \label{fig:1a}
  \end{subfigure}%
    \begin{subfigure}[b]{0.3\textwidth}
    \centering
    \includegraphics[width=\linewidth]{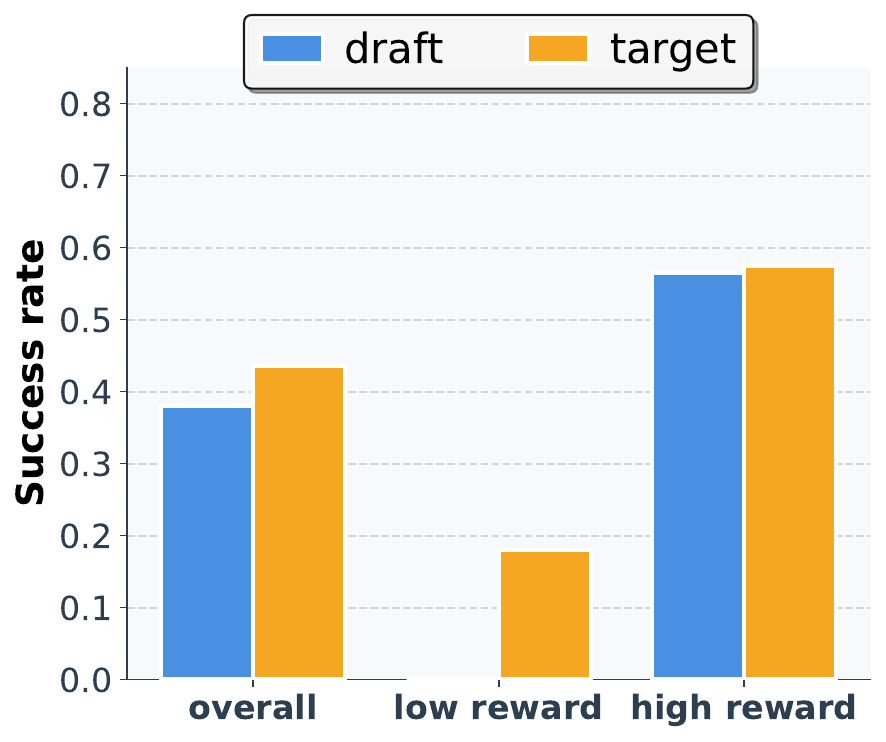}
    \caption{}
    \label{fig:1b}
  \end{subfigure}

  \caption{\textbf{(a)} Latency of generation from the target model (Qwen2.5-7B-Instruct) vs. generation from the draft model (Qwen2.5-1.5B-Instruct) with scoring: we observe that latency savings from using the draft model to generate candidate blocks overcomes the overhead of scoring by the target model and PRM. \textbf{(b)} We generate the first $8$ steps of reasoning from the target model, and complete the remaining steps either using the draft model or the target model. The initial $8$-step partial reasoning traces generated by the target model are bucketed into high reward (PRM score at least $0.5$) and low reward (PRM score at most $0.5$). Using the draft model to complete high reward traces solves a very similar proportion of problems compared to if they were completed using the target model. The performance is dismal ($\approx 0$) when draft model is used to complete low reward traces.\vspace{-0.8em}} 
  \label{fig:explanatory_toys}
\end{figure}

Before presenting the algorithm details, we discuss some general observations that lead to the development of our \specalign algorithm, which could be of independent interest. The most costly step (in terms of latency) for beam search with a large target model is autoregressive generation to obtain the beams (i.e., candidate reasoning blocks). Our main technique for reducing this latency cost is to instead speculatively generate beams via autoregressive generation from a smaller, and faster draft model, which are then scored using the target model and reward model. Due to the fact that autoregressive sampling in modern LLMs is bottlenecked by memory bandwidth, evaluating a reasoning block with a target model (i.e., computing logits of each token) is much faster than generating a reasoning block from the same model, despite taking the same number of FLOPs. This also holds even when combined with the draft model generation time, which is faster due to its smaller size. Moreover, the scoring by the target model and reward model can be executed in parallel to further reduce latency. {While using small draft models for improving latency is the subject of several past work on speculative decoding \citep{leviathan2022fast}, we carry out an experiment in \Cref{fig:1a} to understand the typical per-step latency improvements with modern reasoning and reward models.}

\begin{observation}[\textit{Low-latency draft model generation}] \label{low-latency}
Generating multiple candidate sequences from the draft model, followed by scoring them using the larger target model and a reward model can be significantly faster compared to generating responses directly from the target model.
\end{observation}

\vspace{-0.25em}
While generating reasoning blocks from a draft model is faster, their average quality will be worse since a less capable model is used for generation. Careful draft selection is needed to guarantee a minimal drop in accuracy. Speculative decoding~\citep{leviathan2022fast, sun2023spectr, miao2024specinfer} validates generated drafts by attempting to match the target model distribution. However, in the process, drafts may be unnecessarily discarded, even if they happened to collect high rewards. Thus, it is reasonable to expect that a \emph{soft} verification mechanism, which relaxes the distribution matching guarantee can obtain a speedup, while still maintaining high accuracy on reasoning tasks. 

On the other hand, reasoning datasets typically evaluate model capabilities across a range of skills and difficulty levels. At a finer granularity, certain steps of the solution to a problem might be easy, and a large target model may be overprovisioned to generate steps that the draft model can already solve. For harder steps, it makes sense to defer to the more capable target model for candidate generation.

These discussions motivate the following two questions,

\vspace{-0.5em}
\begin{myenum}
\item 
How can we identify steps where speculative drafting will not significantly hurt accuracy?
\item Given a set of candidate steps
how do we best select candidates based on scores from the target model and the reward model?
\end{myenum}
\vspace{-0.5em}
To gain intuition towards answering the first question, we study how the reward of reasoning trace evolves over the course of a generation, and how it relates to the final accuracy of the generated response. In fig.~\ref{fig:1b} we carry out two inference runs, where we start by generating $8$ steps of reasoning from the target model. In one case, we continue the rest of the trace using the large target model, and in the other we switch to the draft model and continue thereon. When we condition on the traces having high reward (as computed by a PRM) at the 8th step, we observe that completing the trace using the draft model affects accuracy very minimally compared to completing with the target model. This observation motivates switching between the models based on process reward of the trajectory.



\begin{observation}[\textit{Draft model can complete high-reward traces}] \label{obs:high-reward}
Generating blocks from the draft model does not hurt accuracy when the PRM score of a partial trace is high.
\end{observation}
\vspace{-0.75em}

As a counterpart to this observation, we note from \Cref{fig:1b} that continuing generation from the draft model starting from low reward traces performs poorly. To minimize any decrease in accuracy from speculative drafting, it is reasonable to begin with generating steps from the target model until high reward traces are encountered.

\vspace{-0.25em}
\subsection{The \specalign Algorithm}
\vspace{-0.25em}

With the insights developed above, we present our main algorithm \specalign. The algorithm follows an overall draft-and-select framework where the final response is generated block-by-block. In each iteration, the algorithm uses a generation model $\pigen$ to propose candidate blocks $\{ Y_i \}_{i=1}^n$ to extend the current reasoning trace $\currtrace$ through autoregressive sampling. $\pigen$ is initialized to be a larger reasoning model $\pitarget$, and over the course of generation switches to a smaller draft model $\pidraft$. When $\pigen = \pidraft$, we call this a speculative drafting step.

The candidate blocks are then evaluated by both a larger, more capable reasoning model $\pitarget$ (for speculative drafting step), and a reward model $r$. A score (\cref{eq:score1}) is assigned to each candidate based on the evaluations, which is then used to select a candidate block (summarized in \Cref{alg:subsample}). The selected block is appended to the existing response to begin the next iteration.

We use the score to adaptively decide which model to use to generate candidate blocks. This is most closely related to the approach of cascades \citep{dohan2022languagemodelcascades,jitkrittum2023does,yue2024largelanguagemodelcascades,narasimhan2024faster} and aims to ensure that the larger model generates candidates for the harder steps of a problem, while the smaller draft model is used for the easier steps.

\subsection{Candidate selection: \subsample} \label{sec:subsample}

Next we describe the \subsample mechanism which carries out the selection process of candidate blocks $\{ Y_i \}_{i=1}^n$. The goal is to select a response which simultaneously has high reward (i.e., completions are likely to lead to a correct response)
while not having very low probability under the larger target model. \subsample carries this out by combining the evaluations from both the large reasoning model and the reward model in the following score function:

\begin{align}
\vspace{-1em}
    S_i = \log \left( \frac{\pitarget (Y_i|\currtrace)}{\pigen (Y_i|\currtrace)} \right) + \beta r (Y_i|\currtrace)
    \label{eq:score1}
\end{align}
When $\pigen = \pitarget$, this collapses to $r$. With $\pigen = \pidraft$ this also includes the log density ratio of the block between draft and target models. We formalize the choice of score in \Cref{subsec:score} through the lens of optimizing the expected reward with KL regularization (to the target model).

\vspace{-0.25em}
\begin{algorithm}[h!]
   \caption{\specalign~meta algorithm}
   \label{alg:special}
\begin{algorithmic}[1]
   \STATE {\bfseries Input:} Prompt $X$, draft model $\pidraft$,
   \STATE \hspace{2.75em} target model $\pitarget$, PRM $r$,
   \STATE {\bfseries Hyperparameters:} Beam-width $n$,
   \STATEx \hspace{8em} Block-size $\gamma$,
   \STATEx \hspace{8em} Switching threshold $\tau$.
   \STATE {\bfseries Initialize:} Partial response $\currtrace \gets X$
   \STATE \hspace{4.2em} Draft generator $\pigen \gets \pitarget$
   \STATEx {\color{blue} \LineComment{Generation begins with target model.}}
   \STATE {\bfseries Repeat}
   \STATEx \underline{\textit{Step I.}} Generate candidate blocks
   \STATE Draw $n$ blocks $\sim \pigen$ of block-size $\gamma$,
   \vspace{-0.75em}
   \begin{equation*}
       \Ydraft = \{ Y_i \}_{i=1}^n \! \overset{\text{i.i.d.}}{\sim} \! \pigen (\cdot | \currtrace)
   \end{equation*}
   \vspace{-1em}
   \STATEx \underline{\textit{Step II.}} Select candidate block
   \STATE $Y = \subsample( \Ydraft, \beta, \pigen | \currtrace)$,
   \STATE $\currtrace \gets (\currtrace, Y)$ {\color{blue} \LineComment{$\subsample(\cdot)$ selects a}}
   \STATEx {\color{blue} \LineComment{draft using scores in \cref{eq:score1}}}
   \vspace{0.25em}
    \STATEx \underline{\textit{Step III.}} Dynamic switching
    \STATE \algorithmicif \ $\max_{i \in [n]} r(Y_i|\currtrace) > \tau$ \LineLabel{line:9}
   \vspace{0.2em}
    \STATE \hspace{1em} $\pigen \gets \pidraft$  {\color{blue} \LineComment{Switch to generating from draft}}
   \STATEx {\color{blue} \LineComment{model if reward signal $r$ is large}}
   \vspace{0.25em}
   \STATE {\bfseries Until:} Target sequence length is achieved.
  \STATE {\bfseries Return:} $\currtrace$
\end{algorithmic}
\end{algorithm}

\begin{algorithm}[h!]
   \caption{\subsample~subroutine}
   \label{alg:subsample}
\begin{algorithmic}[1]
   \STATE {\bfseries Input:} Context $\currtrace$, Draft blocks $\{ Y_i \}_{i=1}^n$,
   \STATEx \hspace{2.75em} Target model $\pitarget$, PRM $r$,
   \STATEx \hspace{2.75em} Draft generator $\pigen$,
   \STATE {\bfseries Hyperparameters:} Threshold $\tau$, Temp. $\beta$
   \STATE \textbf{Initialize:} $L \gets [n]$
   \STATE \algorithmicfor\ $i=1$ {\bfseries to} $n$,
   \STATE \hspace{0.65em} Define,
   \vspace{-0.75em}
   \begin{equation*}
       S_i = \log \left( \frac{\pitarget (Y_i|\currtrace)}{\pigen (Y_i|\currtrace)} \right) + \beta r (Y_i|\currtrace)
   \end{equation*}
   \STATE Sample $i^\star \in [n]$ w.p., $P(i^\star = i) \propto e^{S_i} \mathbb{I} (i \in L)$.
  \STATE {\bfseries Return:} $Y_{i^\star}$
\end{algorithmic}
\end{algorithm}




\subsection{Deferrals via dynamic switching}\label{sec:cascade}

\specalign uses a deferral rule to determine which (draft or target) model to use for generating candidate blocks in each subsequent iteration. Along the lines of \Cref{obs:high-reward}, the algorithm begins generation from the target model, since the PRM reward is low to begin with. The algorithm proposes to switch generating future reasoning steps from the draft model if high reward partial traces are present in the current set of beams (i.e., their reward exceeds a threshold). This builds on our finding in \Cref{obs:high-reward}: it is beneficial from a latency point-of-view to switch to generating from the draft model when we already have traces with high reward.

\section{A Theoretical Perspective on \specalign}

\subsection{Motivation behind scores in \subsample} \label{subsec:score}


\label{sec:motivation}
The design of the score (\cref{eq:score1}) in \subsample is motivated by the objective of solving a one-step KL-regularized reward maximization problem \citep{jaques2019way,stiennon2020learning}. For the base model $\pitarget$ and a reward function $r$, the objective is to maximize,
\vspace{-.05in}
\begin{align} \label{eq:kl-reg-reward-max}
    \mathcal{L}_\beta (\pi) = \mathbb{E}_{\rho,\pi} [r(Y|X)] - \frac{1}{\beta} \kl(\pi \|\pitarget).
\vspace{-.05in}
\end{align}
The optimal solution to this objective has an explicit form~\citep{korbak2022rl, korbak2022reinforcement,yang2024asymptotics}:
\begin{align} \label{eq:opt-define}
    \pi^\star_\beta (y|x) \propto \pitarget (y|x) e^{\beta r(y|x)}.
\end{align}
For a different generator model $\pigen$,
let $S (y|x) = \log(\pitarget (y|x) / \pigen (y|x) ) + \beta r(y|x)$. Then, the optimal policy can be written down in a different form by importance reweighting. Namely,
\begin{align} \label{eq:opt-redefine}
    \pi^\star_\beta (y|x) \propto \pigen (y|x) e^{S(y|x)}.
\end{align}
Thus, at a high level, any inference-time sampling approach that generates multiple candidate blocks from $\pitarget$ and uses $r$ to subselect one, has a one-to-one mapping to a different approach which generates candidates from $\pidraft$ and instead use $S(\cdot|x)$ for subselection. Alternatively, \cref{eq:kl-reg-reward-max} is equivalent to the \textit{score maximization problem}, subject to a KL-penalty \textit{with respect to the generator model}, $\mathbb{E}_{\rho,\pi} [S(Y|\currtrace)] - \frac{1}{\beta} \kl(\pi \|\pigen)$.

We will defer a more detailed analysis and interpretation of \subsample as an instantiation of sequential monte carlo (SMC) or particle filtering \citep{rubin1987calculation,naesseth2019elements}. This method attempts to sample from the distribution in \cref{eq:opt-redefine} by approximating the partition function, $Z(x) = \mathbb{E}_{Y \sim \pigen (\cdot|x)} [ \exp(S(Y|x))]$, via computing an empirical estimate over a finite set of ``particles''.

\subsection{Theoretical guarantees for \specalign}

In this section, we describe theoretical guarantees for \specalign. The nature of result we will establish is that, under certain assumptions, as the beam-width increases, the distribution of output of \specalign converges to the optimal policy of the KL-regularized reward maximization problem (\cref{eq:kl-reg-reward-max}). First we formalize some notation below.

\paragraph{Notation.} Denote the space of prompts as $\mathcal{X}$, the space of responses as $\mathcal{Y}$. Given a (partial) response $y$, let $(y_1,y_2,\cdots)$ denote its blocks and let $y_{\le t} = (y_1,\cdots,y_t)$ denote the prefix of the first $t$ blocks. By default, we will let $y_{\le 0} = \emptyset$. $H$ denotes the horizon: the maximum number of blocks that the model is allowed to generate. We will assume that there exists a target distribution over prompts, denoted $\rho$, and the notation $\mathbb{E}_{\rho,\pi} [\cdot]$ is shorthand for $\mathbb{E}_{X \sim \rho, Y \sim \pi(\cdot|X)} [\cdot]$, where $Y$ is a full response drawn from $\pi$. The goal is to design an algorithm that samples from  the policy that maximizes the KL-regularized expected reward (\cref{eq:kl-reg-reward-max}).

Next, we introduce the ``idealized'' notion of PRM under which we demonstrate guarantees for \specalign. We begin with the notion of an optimal KL-regularized value function \citep{ziebart2008maximum,zhou2025q}.

\begin{definition}[Optimal KL-regularized value function] \label{def:value}
For a policy $\pi$, and any prompt $x \in \mathcal{X}$ and partial trace $y_{\le t}$, the optimal KL-regularized value of a policy is defined as,
\begin{align*}
    V^{\pi}_\beta (x, y_{\le t}) = \frac{1}{\beta} \log \left( \mathbb{E}_{Y \sim \pi ( \cdot|x,y_{\le t})} \left[ e^{\beta r (x, Y)} \right] \right)
\end{align*}
where $Y$ is a completion of the partial trace $y_{\le t}$. For any full response $y \in \mathcal{Y}$, $V^{\pi}_\beta (x, y) = r(x,y)$.
\end{definition}

\begin{definition}[Idealized Process Reward Model / optimal KL-regularized advantage function] \label{def:prm} Let $\rprm$ denote the PRM corresponding to the \textit{KL-regularized advantage function} with respect to some policy $\pi$. Formally, given any prefix of blocks $y_{\le t-1}$ of length $t-1$ and a subsequent block $y_t$,
\begin{align*}
	\rprm^\pi (y_t |x, y_{\le t-1}) &= V^{\pi}_\beta (x, y_{\le t}) - V^{\pi}_\beta (x, y_{\le t-1})
\end{align*}
where $V^{\pi}_\beta$ is defined in \Cref{def:value}. The idealized PRM we consider is $\rprmtarget = \rprm^\pi$ for $\pi=\pitarget$.
\end{definition}

This above notion of PRM has also been used in recent works \citep{mudgal2023controlled, zhou2025q,brantley2025accelerating} for solving the KL-regularized reward maximization problem. The latter work provides an efficient and scalable approach for learning this PRM from offline datasets, and shows that policies trained by carrying out RL against this PRM perform well empirically, even on low generation budgets. For the purpose of keeping the theoretical presentation cleanest, we analyze \specalign when implemented with an exact version of this PRM. Our main result analyzes \specalign under a subtle modification to make analysis more tractable, which is described in \Cref{app:warm-up}. This pertains to drawing the number of beams generated per step from a Poisson distribution.


\begin{figure*}[t]
    \centering
    \includegraphics[width=0.95\linewidth]{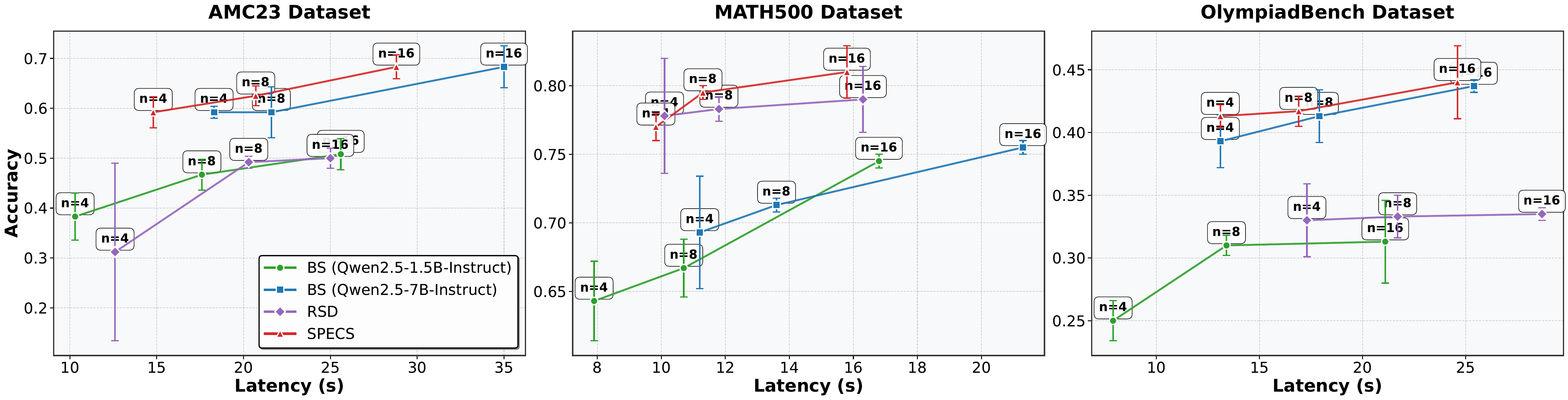}
    \caption{Accuracy vs per-query average latency curves for \specalign and beam search with draft model, target model and RSD. The error bars show the standard deviation computed over 3 independent runs. The draft model is Qwen2.5-1.5B-Instruct and the target model is Qwen2.5-7B-Instruct. }
    \label{fig:accuracy_latency_tradeoff}
\end{figure*}

\begin{theorem} \label{theorem:specalign-block}
Assume $n \ge 3$ and suppose \specalign \ is implemented with the idealized PRM as defined in \Cref{def:prm} with beam-width $n$. Then, for reasoning problems over $H$ blocks, in the finite-block length setting, the policy $\pi_{\specalign}$ returned by \specalign \ satisfies,
\begin{align} \label{eq:thm}
    \kl (\pi_{\specalign} , \pi^\star_\beta) \le \widetilde{\mathcal{O}}_n \left( H \cdot \Cdtstep^2 \frac{e^{2 \beta R}}{n^{3/2}} \right)
\end{align}
Here, we assume that the PRM reward range is $[0,R]$. The block-level coverage coefficient $\Cdtstep$ is,
\begin{align}\label{eq:Cdtstep}
    \Cdtstep = \sup_{t \ge 0} \sup_{y_{\le t}} \left\| \frac{\pitarget (\cdot|x, y_{\le t})}{\pidraft (\cdot|x, y_{\le t})} \right\|_\infty \left\| \frac{\pidraft (\cdot|x, y_{\le t})}{\pitarget (\cdot|x , y_{\le t})} \right\|_\infty \nonumber
\end{align}
where $\pi (\cdot|x,y_{\le t})$ denotes the next-block distribution given the partial trace $y_{\le t}$ and prompt $x$.
\end{theorem}

The proof of this result is deferred to \Cref{app:theory}.

\paragraph{Additional results on robustness to PRM noise.} Our theoretical analysis can be further extended to cases where the PRM is not idealized (cf. \Cref{def:prm}). In \Cref{sec:apx-prm} we show that \specalign's guarantees gracefully degrades with the (suitably defined) approximation noise level of the PRM.


\vspace{-.1in}\section{Experiments} \label{sec:experiments}



In this section, we present empirical evaluation results on the \emph{latency-accuracy tradeoff} of \specalign and compare it with test-time scaling baselines in the literature.

\subsection{Experimental Setup} \label{sec:expsetup}

\textbf{Datasets:} We conduct experiments on a diverse suite of reasoning benchmarks. For mathematical reasoning, we use AMC23~\citep{faires2022contest}, MATH500~\citep{hendrycks2021measuring}, and OlympiadBench~\citep{he2024olympiadbench}. The AMC23 dataset comprises 40 problems testing algebra, geometry, number theory, and combinatorics. On the other hand, the MATH500 benchmark evaluates mathematical problem-solving, while OlympiadBench contains challenging olympiad-level problems requiring deep reasoning. Following prior works \citep{qiu2024treebon, zhang2024accelerating}, we use a randomly selected subset of 100 questions from MATH500 and OlympiadBench to expedite evaluation and report average performance across multiple seeds. We also evaluate 
 \specalign on non-mathematics datasets using the GPQA benchmark~\citep{rein2024gpqa}. GPQA consists of 448 high-quality multiple-choice questions in biology, physics, and chemistry, written by domain experts. 

\textbf{Models:} We use several models from the Qwen and Llama families as our base generation models. We evaluate over two draft-target model pairs, Qwen2.5-1.5B-Instruct and Qwen2.5-7B-Instruct from the Qwen family of models~\citep{yang2024qwen2}, as well as Llama3.2-1B-Instruct and Llama-3.1-8B-Instruct~\citep{dubey2024llama}. These models vary significantly in capability, allowing us to test the scalability of our approach. For reward-guided selection in both beam search and \specalign, we use the Skywork-o1-Open-PRM-Qwen-2.5-7B~\citep{he_2024_16998085}, which is specifically trained for mathematical reasoning tasks. It is worth pointing out that the generator (draft/target) models have not been specifically finetuned for math reasoning tasks. 

\textbf{Baselines:} 
Our primary baseline is standard beam search guided by the same PRM used within \specalign. 
This involves generating multiple candidate sequences (beams) block-by-block and selecting the highest scoring sequences according to the reward model \citep{snell2025scaling,beeching2024scalingtesttimecompute}
We compare \specalign against beam search using the target model (high latency/accuracy) and the draft model (low latency/accuracy). 
We also compare against Reward-Guided Speculative Decoding (RSD)~\citep{liao2025rewardguidedspeculativedecodingefficient}, which uses PRM signals for rejection sampling. We implement a stronger multi-beam variant of the standard single-beam method: we select the highest-scoring of $n$ draft beams, but if its reward falls below the recommended threshold of $0.7$, we instead select the best of $n$ beams generated by the target model. We call this implementation RSD++.

\textbf{Evaluation Metrics:} We report accuracy (percentage of correctly solved problems) and average wall-clock latency per problem. Our primary focus is to optimize the pareto frontier of the latency-accuracy trade-off.

\vspace{-.1in}\subsection{Results}

\paragraph{Performance on Mathematical Reasoning.}
As shown in Table \ref{tab:results_table}, for a fixed beam-width $n$, \specalign achieves accuracy comparable to or exceeding the strong baseline of beam search with the large model. This performance gain is accompanied by up to $\sim$\textit{18\%} latency savings, demonstrating a favorable trade-off. 
The percentage of steps generated by the large model increases for the harder datasets (AMC23, OlympiadBench), indicating that \specalign is able to adapt the usage of the large model to the hardness of questions.


\begin{table*}[h!]
\caption{Performance comparison of methods across datasets and beam search settings. Here $n$ denotes per-step beam width, `Lat. (s)' is average per problem latency in seconds and `\% \texttt{\textcolor{darkred}{B}}' represents the percentage of steps generated by the big model \texttt{\textcolor{darkred}{B}} averaged across problems. (\texttt{\textcolor{darkgreen}{S}}: small model, \texttt{\textcolor{darkred}{B}}: big model). ($\star$) we implement a much stronger variant of RSD in our experiments for beam search that we call RSD++ as a competitive baseline to compare \specalign against; more details are provided in \Cref{sec:expsetup}.}
\centering
\vspace{0.2em}
\resizebox{0.94\linewidth}{!}{
\begin{tabular}{c|l|cccc|cccc}
\toprule
\textbf{Method} & $n$ & \textbf{Acc} (\%) $\uparrow$ & \textbf{Lat} (s) $\downarrow$ & \textbf{Lat/step} (s) $\downarrow$ & \textbf{\%} \texttt{\textcolor{darkred}{B}} & \textbf{Acc} (\%) $\uparrow$ & \textbf{Lat} (s) $\downarrow$ & \textbf{Lat/step} (s) $\downarrow$ & \textbf{\%} \texttt{\textcolor{darkred}{B}} \\
\cmidrule(lr){1-10}
 &  & \multicolumn{4}{c}{Qwen2.5 Family} & \multicolumn{4}{c}{Llama3 Family} \\
\cmidrule(lr){1-10}
\multicolumn{2}{c}{} & \multicolumn{8}{c}{AMC23} \\
\cmidrule(lr){1-10}
\Smodel & \color{gray}4  & \color{gray}38.3 $\pm$ 4.7 & \color{gray}10.3 $\pm$ 0.6 & \color{gray}0.80 $\pm$ 0.01 & \color{gray}0  
& \color{gray}11.7 $\pm$ 5.1 & \color{gray}5.94 $\pm$ 0.24 & \color{gray}0.38 $\pm$ 0.15 & \color{gray}0 \\
\Bmodel & 4               
& 59.2 $\pm$ 1.2            & 18.3 $\pm$ 1.5            & 1.46 $\pm$ 0.06           & 100
& 26.3  $\pm$ 1.3 & 12.7 $\pm$ 1.3 & 0.59 $\pm$ 0.17  & 100 \\
RSD++ & 4               
& 31.2 $\pm$ 17.8 & 12.6 $\pm$ 3.8 & 0.82 $\pm$ 0.60 & 60.5 $\pm$ 1.4 
& 15.0 $\pm$ 7.4            & 8.9 $\pm$ 1.1            & 0.66 $\pm$ 0.79           & 91.8 $\pm$ 0.6 \\
\graycell{\specaligntable} & \graycell{4}  
& \graycell{\textbf{59.2 $\pm$ 3.1}} & \graycell{14.8 $\pm$ 0.5} & \graycell{1.42 $\pm 0.13$}               & \graycell{67.9 $\pm$ 1.8}
& \graycell{\textbf{26.3 $\pm$ 5.3}} & \graycell{15.4 $\pm$ 1.8} & \graycell{0.65 $\pm$ 0.03} & \graycell{96.4 $\pm$ 1.2} \\
\cmidrule(lr){1-10}
\Smodel & \color{gray}8  
& \color{gray}46.7 $\pm$ 3.1 & \color{gray}17.6 $\pm$ 2.3 & \color{gray}1.14 $\pm$ 0.01 & \color{gray}0
& \color{gray}14.2 $\pm$ 4.2 & \color{gray}11.85 $\pm$ 0.22 & \color{gray}0.61 $\pm$ 0.46 & \color{gray}0 \\
\Bmodel & 8              
& 59.2 $\pm$ 5.1            & 21.6 $\pm$ 1.1            & 2.00 $\pm$ 0.06           & 100
 & 27.5 $\pm$ 2.0 & 19.1 $\pm$ 0.96 & 0.84 $\pm$ 1.39 & 100
 \\
RSD++ & 8               
& 49.2 $\pm$ 1.2 & 20.3 $\pm$ 1.2 & 1.99 $\pm$ 4.43 & 61.1 $\pm$ 0.3
& 18.3 $\pm$ 5.9            & 10.8 $\pm$ 1.9           & 0.91 $\pm$ 0.8           & 90.0 $\pm$ 1.1 \\
\graycell{\specaligntable} & \graycell{8}  
& \graycell{\textbf{62.5 $\pm$ 2.0}} & \graycell{20.7 $\pm$ 2.0} & \graycell{1.94 $\pm 0.11$}               & \graycell{62.8 $\pm$ 2.1}
& \graycell{\textbf{30.0 $\pm$ 2.5}} & \graycell{28.0 $\pm$ 1.2} & \graycell{1.04 $\pm$ 0.05} & \graycell{93.6 $\pm$ 1.0} \\
\cmidrule(lr){1-10}
\Smodel & \color{gray}16 
& \color{gray}50.8 $\pm$ 3.1 & \color{gray}25.6 $\pm$ 3.1 & \color{gray}1.81 $\pm$ 0.03 & \color{gray}0
& \color{gray}24.2 $\pm$ 4.7 & \color{gray}24.47 $\pm$ 1.21 & \color{gray}1.13 $\pm$ 0.95 & \color{gray}0\\
\Bmodel & 16              
& 68.3 $\pm$ 4.2            & 35.0 $\pm$ 2.6            & 2.80 $\pm$ 0.13           & 100
& 32.5 $\pm$ 4.7 & 32.1 $\pm$ 2.4 & 1.24 $\pm$ 1.94 & 100 \\
RSD++ & 16               
& 50.0 $\pm$ 2.0 & 25.0 $\pm$ 0.8 & 2.56 $\pm$ 1.78 & 55.6 $\pm$ 2.2
& 20.0 $\pm$ 4.1            & 16.3 $\pm$ 0.4           & 1.18 $\pm$ 1.09           & 90.0 $\pm$ 2.0 \\
\graycell{\specaligntable} & \graycell{16} 
& \graycell{\textbf{68.3 $\pm$ 2.4}} & \graycell{28.8 $\pm$ 3.5} & \graycell{2.70 $\pm$ 0.05}               & \graycell{70.3 $\pm$ 1.7}
& \graycell{\textbf{36.3 $\pm$ 3.8}} & \graycell{44.2 $\pm$ 2.9} & \graycell{1.56 $\pm$ 0.10} & \graycell{98.1 $\pm$ 1.1} \\

\cmidrule(lr){1-10}
\multicolumn{2}{c}{} & \multicolumn{8}{c}{MATH500} \\
\cmidrule(lr){1-10}
\Smodel &  \color{gray}4  
& \color{gray}64.3 $\pm$ 2.9 &  \color{gray}7.9 $\pm$ 1.0  & \color{gray}0.65 $\pm$ 0.04 & \color{gray}0
& \color{gray}23.1 $\pm$ 1.5 &  \color{gray}4.82 $\pm$ 0.40  & \color{gray}0.27 $\pm$ 3.95 & \color{gray}0 \\
\Bmodel & 4                
& 69.3 $\pm$ 4.1             & 11.2 $\pm$ 3.7             & 1.13 $\pm$ 0.01            & 100
& 32.0 $\pm$ 3.3 & 8.4 $\pm$ 0.1 & 0.39 $\pm$ 0.29 & 100 \\
RSD++ & 4               
& \textbf{77.8 $\pm$ 4.2} & 10.1 $\pm$ 0.2 & 1.14 $\pm$ 0.52 & 32.3 $\pm$  0.6
& \textbf{35.7 $\pm$ 4.1}            & 11.2 $\pm$ 0.6            & 0.55 $\pm$ 2.15           & 72.7 $\pm$ 1.1 \\
\graycell{\specaligntable} & \graycell{4}  
& \graycell{{77.0 $\pm$ 1}} & \graycell{9.83 $\pm$ 0.5}  & \graycell{1.11 $\pm$ 0.14} & \graycell{49.0 $\pm$ 0.8}
& \graycell{{35.3 $\pm$ 3.9}} & \graycell{10.8 $\pm$ 0.8} & \graycell{0.40 $\pm$ 1.17} & \graycell{91.1 $\pm$ 0.9} \\
\cmidrule(lr){1-10}
\Smodel &  \color{gray}8  
& \color{gray}66.7 $\pm$ 2.1 &  \color{gray}10.7 $\pm$ 0.6 & \color{gray}1.01 $\pm$ 0.03 & \color{gray}0
& \color{gray}28.3 $\pm$ 1.7 &  \color{gray}9.55 $\pm$ 0.40 & \color{gray}0.45 $\pm$ 0.68 & \color{gray}0 \\
\Bmodel & 8                
& 71.3 $\pm$ 0.5             & 13.6 $\pm$ 1.9             & 1.51 $\pm$ 0.00            & 100
& 33.0 $\pm$ 2.2  & 12.1 $\pm$ 0.1 & 0.52 $\pm$ 0.3 & 100 \\
RSD++ & 8               
& 78.3 $\pm$ 0.9 & 11.8 $\pm$ 0.3 & 1.37 $\pm$ 0.83 & 28.4 $\pm$ 0.2 
& 39.3 $\pm$ 2.1            & 15.2 $\pm$ 0.5            & 0.71 $\pm$ 2.02           & 73.9 $\pm$ 1.2 \\
\graycell{\specaligntable} & \graycell{8}  
& \graycell{\textbf{79.5 $\pm$ 0.5}} & \graycell{11.3 $\pm$ 0.6} & \graycell{1.31 $\pm$ 0.56} & \graycell{45.5 $\pm$ 0.3}
& \graycell{\textbf{40.7 $\pm$ 2.6}} & \graycell{15.2 $\pm$ 0.74} & \graycell{0.57 $\pm$ 0.09} & \graycell{85.2 $\pm$ 1.6} \\
\cmidrule(lr){1-10}
\Smodel &  \color{gray}16 
&  \color{gray}74.5 $\pm$ 0.5 &  \color{gray}16.8 $\pm$ 1.9 & \color{gray}1.57 $\pm$ 0.10 &  \color{gray}0
&  \color{gray}34.3 $\pm$ 0.6 &  \color{gray}17.42 $\pm$ 0.69 & \color{gray}0.78 $\pm$ 2.42 &  \color{gray}0 \\
\Bmodel & 16               
& 75.5 $\pm$ 0.5             & 21.3 $\pm$ 0.5             & 2.15 $\pm$ 0.02            & 100
& $39.0 \pm 2.2$ & $18.9 \pm 0.46$ & $0.79 \pm 0.21$ & 100 \\
RSD++ & 16               
& 79.0 $\pm$ 2.4 & 16.3 $\pm$ 0.4 & 1.84 $\pm$ 0.64 & 26.8 $\pm$ 0.3
& 40.0 $\pm$ 2.2            & 23.0 $\pm$ 1.1            & 1.02 $\pm$ 1.98           & 70.3 $\pm$ 1.3 \\
\graycell{\specaligntable} & \graycell{16} 
& \graycell{\textbf{81.0 $\pm$ 1.9}} & \graycell{15.8 $\pm$ 0.06} & \graycell{1.31 $\pm$ 0.04} & \graycell{45.5 $\pm$ 0.8}
& \graycell{\textbf{44.0 $\pm$ 0.8}} & \graycell{22.9 $\pm$ 6.3} & \graycell{0.86 $\pm$ 0.19} & \graycell{78.1 $\pm$ 1.1} \\

\cmidrule(lr){1-10}
\multicolumn{2}{c}{} & \multicolumn{8}{c}{OlympiadBench} \\
\cmidrule(lr){1-10}
\Smodel & \color{gray}4  
& \color{gray}25.0 $\pm$ 1.6 & \color{gray}7.9 $\pm$ 0.2  & \color{gray}0.81 $\pm$ 0.2 & \color{gray}0
& \color{gray}8.0 $\pm$ 1.4 & \color{gray}5.58 $\pm$ 0.12  & \color{gray}0.35 $\pm$ 0.26 & \color{gray}0\\
\Bmodel & 4               
& {39.3 $\pm$ 2.1}   & 13.1 $\pm$ 0.3             & 1.32 $\pm$ 0.3 & 100
& $17.0 \pm 1.2$ & $13.0 \pm 0.05$ & $0.59 \pm 0.18 $ & 100 \\
RSD++ & 4               
& 33.0 $\pm$ 2.9 & 17.3 $\pm$ 0.4 & 1.60 $\pm$ 0.70 & 66.3 $\pm$ 0.8
& 7.3   $\pm$ 1.7         & 13.1 $\pm$ 0.4            & 0.61 $\pm$ 0.69           & 95.2 $\pm$ 0.9 \\
\graycell{\specaligntable} & \graycell{4}  
& \graycell{\textbf{41.3 $\pm$ 0.9}} & \graycell{13.1 $\pm$ 0.3} & \graycell{1.30 $\pm$ 3.4} & \graycell{83.1 $\pm$ 0.4}
& \graycell{\textbf{18.0 $\pm$ 2.8}} & \graycell{14.9 $\pm$ 2.8} & \graycell{0.62 $\pm$ 0.03} & \graycell{98.1 $\pm$ 0.2} \\
\cmidrule(lr){1-10}
\Smodel & \color{gray}8  
& \color{gray}31.0 $\pm$ 0.8 & \color{gray}13.4 $\pm$ 0.3 & \color{gray}1.22 $\pm$ 0.3 & \color{gray}0
& \color{gray}8.0 $\pm$ 3.3 & \color{gray}11.09 $\pm$ 0.49 & \color{gray}0.59 $\pm$ 0.44 & \color{gray}0\\
\Bmodel & 8               
& {41.3 $\pm$ 2.1}   & 17.9 $\pm$ 0.5             & 1.81 $\pm$ 0.5 & 100 
& $18.3 \pm 0.48$ & 18.9 $\pm$ 0.1 & $0.83 \pm 0.14 $ & 100 \\
RSD++ & 8               
& 33.3 $\pm$ 1.7 & 21.7 $\pm$ 1.3 & 1.96 $\pm$ 1.78 & 65.3 $\pm$ 1.2
& 8.0   $\pm$ 1.4         & 18.9 $\pm$ 0.7            & 0.83 $\pm$ 0.56           & 92.9 $\pm$ 1.1 \\
\graycell{\specaligntable} & \graycell{8}  
& \graycell{\textbf{41.7 $\pm$ 1.2}} & \graycell{16.9 $\pm$ 0.5} & \graycell{1.80 $\pm$ 1.5} & \graycell{80.3 $\pm$ 0.4}
& \graycell{\textbf{18.5 $\pm$ 0.5}} & \graycell{19.7 $\pm$ 0.14} & \graycell{0.78 $\pm$ 0.66} & \graycell{94.6 $\pm$ 0.1} \\
\cmidrule(lr){1-10}
\Smodel & \color{gray}16 
& \color{gray}31.3 $\pm$ 3.3 & \color{gray}21.1 $\pm$ 1.1 & \color{gray}2.02 $\pm$ 1.1 & \color{gray}0
& \color{gray}11.3 $\pm$ 1.2 & \color{gray}23.17 $\pm$ 0.73 & \color{gray}1.06 $\pm$ 1.23 & \color{gray}0\\
\Bmodel & 16              
& {43.7 $\pm$ 0.5}   & 25.4 $\pm$ 0.9             & 2.57 $\pm$ 0.9 & 100
& \textbf{22.0} $\pm$ 1.2 & 30.1 $\pm$ 0.8 & $1.20 \pm 5.3$ & 100 \\
RSD++ & 16               
& 33.5 $\pm$ 0.5 & 28.7 $\pm$ 1.3 & 2.58 $\pm$ 13.29 & 61.2 $\pm$ 1.3
& 9.5   $\pm$ 0.5         & 25.8 $\pm$ 0.5            & 1.18 $\pm$ 2.71           & 93.2 $\pm$ 0.8 \\
\graycell{\specaligntable} 
& \graycell{16} 
& \graycell{\textbf{44.0 $\pm$ 2.9}} & \graycell{24.6 $\pm$ 1.0} & \graycell{2.56 $\pm$ 6.2} & \graycell{61.5 $\pm$ 0.3} 
& \graycell{{20.0 $\pm$ 5.7}} & \graycell{31.7 $\pm$ 2.4} & \graycell{1.21 $\pm$ 0.01} & \graycell{91.1 $\pm$ 1.3} \\
\bottomrule
\end{tabular}
}
\label{tab:results_table}
\vspace{-.1in}
\end{table*}

\vspace{-0.75em}
\paragraph{Pareto-frontier of accuracy vs. latency.} 

Varying $\tau$ in \specalign (Line~\ref{line:9}, \Cref{alg:special}) controls the reward threshold for switching to the draft model. Higher values of $\tau$ result in a method which switches to the draft model more quickly, implying lower latencies and lower accuracies, while lower values of $\tau$ results in blocks being generated for longer by the target model (i.e., higher latency and accuracy). Varying $\tau$ results in an accuracy-latency pareto curve. As shown in \Cref{fig:pareto} for MATH500 (using Qwen2.5-1.5B/7B-Instruct), \specalign offers a superior tradeoff, with all baselines (RSD and beam search) falling within the convex hull of its curve.

\begin{figure}[h!]
    \centering
    \includegraphics[width=\linewidth,trim={1.3em 1em 1.05em 0},clip]{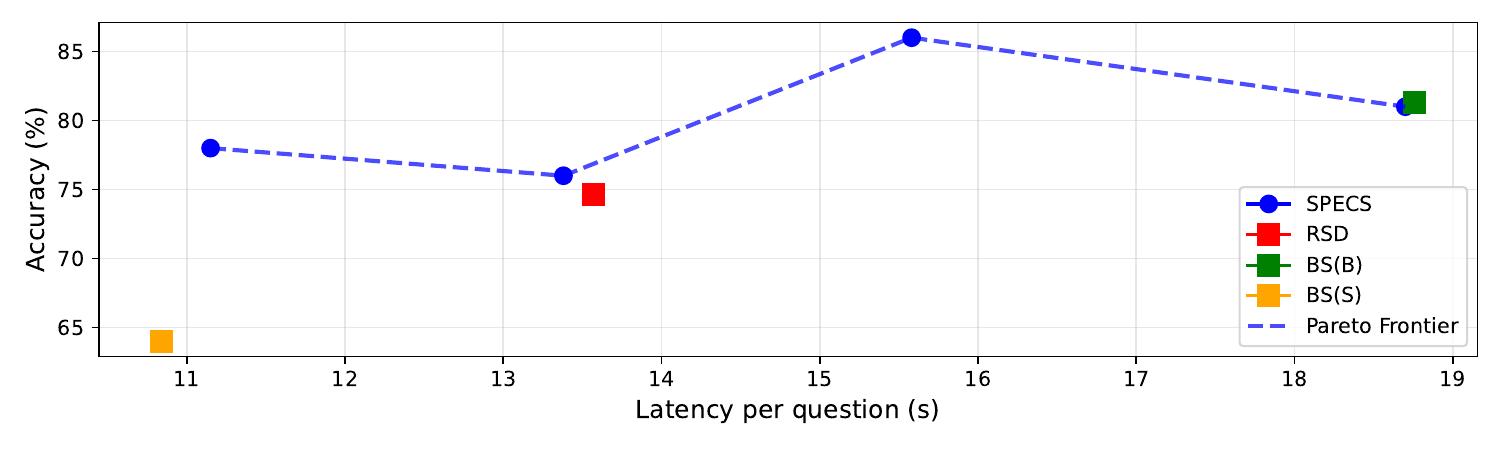}
    \caption{Changing the value of the threshold $\tau$ gives us an accuracy-latency pareto curve for \specalign. As a sanity check, when $\tau \to 1$, \specalign collapses to beam search with the large model, and the corresponding points coincide. }
    \label{fig:pareto}
\end{figure}

\vspace{-1em}
\paragraph{Generalization to Domains Beyond Math.}
Using the Qwen family of models mentioned, we conduct experiments on the GPQA benchmark in \Cref{tab:gpqa_specs}. \specalign achieves better latency compared to beam search with the big model while achieving similar accuracy results, showing the applicability of \specalign to reasoning domains other than mathematics.


{
\begin{table}[h!]
\caption{\specalign results on the GPQA benchmark. Here $n$ denotes the per-step beam width.}
\centering
\resizebox{\linewidth}{!}{
\begin{tabular}{c|cc|cc|cc}
\toprule
$n$ 
& \multicolumn{2}{c}{\textbf{\specalign}} 
& \multicolumn{2}{c}{\textbf{BS (Big Model)}} 
& \multicolumn{2}{c}{\textbf{BS (Small Model)}} \\
\cmidrule(lr){1-7}
& \textbf{Acc} (\%) $\uparrow$ & \textbf{Lat} (s) $\downarrow$
& \textbf{Acc} (\%) $\uparrow$ & \textbf{Lat} (s) $\downarrow$
& \textbf{Acc} (\%) $\uparrow$ & \textbf{Lat} (s) $\downarrow$ \\
\cmidrule(lr){1-7}
4  & 30.8 & 9.00  & 34.2 & 18.6 & 17.8 & 7.77  \\
8  & 32.2 & 21.83 & 32.5 & 23.1 & 22.6 & 10.31 \\
16 & 34.5 & 22.42 & 33.3 & 28.3 & 24.2 & 15.08 \\
\bottomrule
\end{tabular}
}
\label{tab:gpqa_specs}
\end{table}
\arrayrulecolor{black}
\vspace{-10pt}
}

\subsection{Analysis and Ablations}
\label{sec:design_choices}

Here, we present our analysis on the importance of different components of \specalign and its robustness to hyperparameter choices. We defer an analysis of the robustness to noisy reward signals, as well as the latency breakdown and implementation specifics such as the hardware setup and hyperparameter choices to \Cref{sec:further_experiments}.

\paragraph{Effect of log-likelihood based draft selection score.} First, we compare \specalign with a version of the algorithm where we remove dynamic switching, and generate all steps from the draft model with soft verification (in the sense of \cref{eq:score1}). This can alternately be seen as using $\pigen \gets \pidraft$ always. We observe much higher accuracy compared to beam search with the small model (BS(S)) while the latency is not much worse (shown in \Cref{tab:results_table}). This demonstrates the benefit of our score design in \cref{eq:score1}. The results of this ablation are in \Cref{tab:specs_ablation_full}, under ``Only Small Gen''. We also consider a variant of \specalign where we start with the small model first, switching to the larger model if none of the generations of the small model exceed the threshold $\tau$. We show that this variant outperforms RSD++ in \Cref{tab:results_table}, showing the utility of our scoring mechanism over hard best-of-n sampling. This is indicated with ``\specalign ($\pidraft$ start)'' in \Cref{tab:specs_ablation_full}.
We implement one final variant, ``\specalign (no LL)'', which uses no log-likelihood term in the score function (\cref{eq:score1}). This variant performs worse than \specalign, showing the utility of the log-likehood ratio in the score.

\vspace{-1em}
\paragraph{Effect of dynamic switching.} We consider a variant of \specalign where we switch randomly between the draft and target model for generation, in such a way that the average proportion of steps the target model generates matches with the target model usage of \specalign. This results in a method with similar average per-step latencies. We observe that random switching achieves much lower accuracy than dynamic switching based on the reward signal, demonstrating the advantage of our dynamic switching scheme. The results of this ablation are in \Cref{tab:specs_ablation_full}, under ``Random Switch''. Moreover, the fact that \specalign outperforms \specalign ($\pidraft$ start) also shows the advantage of starting by generating from the large model.

\begin{table}[htbp]
\caption{Performance comparison of ablation studies across different datasets and beam search settings using 
 Qwen2.5-Instruct models. `\specalign' is the main method, 'Only Small Gen' generates all beams by the small model, and `Random Switch' switches randomly between small and big model with the same probability as the target model of corresponding \specalign runs,
`\specalign (start with small model)' starts generating with small model and then switches to the target model and `\specalign (no LL)' removes the log-likelihood term from the score of \specalign.
}

\centering
\resizebox{\linewidth}{!}{
\begin{tabular}{l|c|cc|cc|cc}
\toprule
\textbf{Method} & $n$ & \multicolumn{2}{c}{\textbf{AMC23}} & \multicolumn{2}{c}{\textbf{MATH500}} & \multicolumn{2}{c}{\textbf{OlympiadBench}} \\
\cmidrule(lr){1-8}
 & & \textbf{Acc} (\%) $\uparrow$ & \textbf{Lat} (s) & \textbf{Acc} (\%) $\uparrow$ & \textbf{Lat} (s) $\downarrow$ & \textbf{Acc} (\%) $\uparrow$ & \textbf{Lat} (s) \\
\cmidrule(lr){1-8} 
\specalign & 4 & 59.2 & 14.80 & 77.0 & 9.83 & 41.3 & 13.10 \\
Only Small Gen & 4 & 51.2 & 13.68 & 65.0 & 7.45 & 41.1 & 13.00 \\
Random Switch & 4 & 42.2 & 10.86 & 70.0 & 8.59 & 30.1 & 9.43 \\
\specalign ($\pidraft$ start) & 4 & 58.9 & 18.5 & 77.5 & 13.0 & 41.0 & 15.3 \\
\specalign (no LL) & 4 & 43.2 & 12.5 & 71.5 & 8.73 & 32.5 & 16.6 \\
\cmidrule(lr){1-8}
\specalign & 8 & 62.5 & 20.70 & 79.5 & 11.3 & 41.7 & 16.90\\
Only Small Gen & 8 & 56.2 & 19.69 & 66.0 & 8.15 & 39.0 & 16.24 \\
Random Switch & 8 & 50.6 & 15.45 & 70.0 & 12.13 & 33.5 & 13.86 \\
\specalign ($\pidraft$ start) & 8 & 61.4 & 21.9 & 78.9 & 12.7 & 41.3 & 18.1 \\
\specalign (no LL) & 8 & 47.5 & 17.9 & 73.3 & 10.6 & 32.8 & 23.8 \\
\cmidrule(lr){1-8}
\specalign & 16 & 68.3 & 28.80 & 81.0 & 11.748 & 44.0 & 24.60 \\
Only Small Gen & 16 & 63.8 & 26.72 & 73.0 & 11.6 & 39.5 & 23.84 \\
Random Switch & 16 & 54.7 & 23.60 & 74.0 & 11.3 & 38.4 & 21.83 \\
\specalign ($\pidraft$ start)  & 16 & 67.9 & 29.2 & 81.0 & 17.2 & 44.2 & 27.8 \\
\specalign (no LL) & 16 & 48.6 & 28.1 & 74.5 & 14.5 & 33.1 & 27.1 \\
\bottomrule
\end{tabular}
}
\label{tab:specs_ablation_full}
\end{table}

\vspace{-0.5em}
\paragraph{Hyperparameter Sensitivity.}
We also examine whether \specalign is sensitive to hyperparameter tuning. We find that the method is highly stable: setting $\beta = 2L$ (where $L$ is the token budget) and restricting $\tau$ to a coarse set $\{0.7, 0.8, 0.9\}$ recovers nearly all the performance of an exhaustive search. As shown in Table \ref{tab:constraint_hyperparam_selection}, the ``Constrained Search" results differ from the optimal search by only marginal amounts, confirming that \specalign is easy to tune and deploy.

{%

\begin{table}[h]
\caption{Hyperparameter sensitivity analysis. We compare the optimal performance found via exhaustive search against a constrained search with fixed heuristics ($\beta=2L$, coarse $\tau$). The performance drop is negligible, indicating robustness to hyperparameter choices. \specalign (E) means Exhaustive search, \specalign (C) means Constrained search.}
\centering
\resizebox{\linewidth}{!}{
\begin{tabular}{l|c|cc|cc|cc}
\toprule
\textbf{Method} & $n$ & \multicolumn{2}{c}{\textbf{AMC23}} & \multicolumn{2}{c}{\textbf{MATH500}} & \multicolumn{2}{c}{\textbf{OlympiadBench}} \\
\cmidrule(lr){1-8}
& & \textbf{Acc} (\%) $\uparrow$ & \textbf{Lat} (s) $\downarrow$ & \textbf{Acc} (\%) $\uparrow$ & \textbf{Lat} (s) $\downarrow$ & \textbf{Acc} (\%) $\uparrow$ & \textbf{Lat} (s) $\downarrow$ \\
\cmidrule(lr){1-8}
\specalign (E) & 4 & 59.2 & 14.8 & 77.0 & 9.83 & 41.3 & 13.1 \\
\specalign (C) & 4 & 53.1 & 13.3 & 77.0 & 9.80 & 38.5 & 12.5 \\
\midrule
\specalign (E) & 8 & 62.5 & 20.7 & 79.5 & 11.3 & 41.7 & 16.9 \\
\specalign (C) & 8 & 58.4 & 20.2 & 75.3 & 11.0 & 41.0 & 18.1 \\
\midrule
\specalign (E) & 16 & 68.3 & 28.8 & 81.0 & 15.8 & 44.0 & 24.6 \\
\specalign (C) & 16 & 64.8 & 26.4 & 81.0 & 15.8 & 42.6 & 23.8 \\
\bottomrule
\end{tabular}
}
\label{tab:constraint_hyperparam_selection}
\vspace{-1.5em}
\end{table}
\arrayrulecolor{black}
}

\section{Conclusion}

In this work, we address the critical challenge of balancing latency and accuracy in test-time scaling for large language models. We propose \specalign, a latency-aware test-time scaling method for large language models (LLMs) that significantly improves the latency-accuracy tradeoff through speculative drafting with a fast draft model. 
We propose novel integration strategies to merge the draft model, target model, and the PRM, including reward-guided soft verification and a reward-based switching mechanism, to dynamically balance computational resources with response latency. Empirically, we evaluate \specalign across the MATH500, AMC23, OlympiadBench, and GPQA datasets, consistently achieving accuracy comparable to or exceeding  beam search methods while delivering substantial reductions in latency - up to approximately 18\%. Our theoretical analysis confirms that \specalign converges to an optimum of the KL-regularized reinforcement learning objective as the beam width increases.

\bibliography{references}
\bibliographystyle{icml2025}

\newpage
\appendix
\onecolumn

\startcontents[apx]

\section*{Appendix Contents}
\printcontents[apx]{}{1}{\setcounter{tocdepth}{2}}

\section{Related work} \label{sec:related}

\paragraph{Reward-guided search.} Our algorithm is related to the general literature of process reward modeling~\citep{uesato2022solvingmathwordproblems,lightman2023letsverifystepstep,wang2024mathshepherdverifyreinforcellms,snell2025scaling,beeching2024scalingtesttimecompute,zheng2024processbenchidentifyingprocesserrors,zhang2025lessonsdevelopingprocessreward} and reward-guided search~\citep{mudgal2023controlled, beeching2024scalingtesttimecompute,sun2024fastbestofndecodingspeculative,qiu2024treebon,snell2025scaling}. Among these, the closest to our is  reward-guided speculative decoding \citep{liao2025rewardguidedspeculativedecodingefficient}, which also uses a draft model to propose each reasoning step and uses a process reward model to decide on whether to keep the proposal. \specalign~generalizes this by proposing multiple candidate reasoning and combine signals from both the large model and reward model to decide on the candidate reasoning step to keep. Empirically, we observe our method achieves better latency-accuracy tradeoff compared to reward-guided speculative decoding. 

We would like to mention that existing work that utilizes speculative drafts to reduce latency all start with generations from a speculative model, and switches to a larger when the reward signal is low. While our algorithm  starts with the large model and switches to speculative drafts only when the reward signal is high. This approach reduces the time wasted for speculative drafting for low-reward traces, resulting in overall better latency-accuracy tradeoff.  This builds on the novel insight that the accuracy gap between  a small model and a large model is small for these traces, which could be of independent interest.

\paragraph{Tree-based speculative decoding.} Obtaining one sequence from multiple draft sequences has been studied in the speculative decoding literature with the goal of maximizing the likelihood/length of draft tokens accepted while maintaining the same distribution as the large model~\citep{sun2023spectr,miao2024specinfer, cai2024medusa, wang2025opt, chen2024sequoia}. Compared to these, the goal of our soft verification algorithm is to maximize the success rate of downstream applications such as math-reasoning as guided by both the target model and the reward model. Other soft verification methods for speculative decoding \citep{wang2025speculatecollaboratefusingknowledge} consider token-level verification while we focus on step-level candidate selection.

\paragraph{Formal analysis of best-of-$N$ variants.} The theoretical analysis of our soft verification algorithm is related to \citep{verdun2025softbestofnsamplingmodel,huang2025bestofnbestthemcoverage,rohatgi2025tamingimperfectprocessverifiers}, which consider alternative candidate selection methods other than Top-$1$ selection. Our analysis captures the convergence rate in presence of draft models and multiple decoding steps.

\paragraph{Cascading.} Cascading techniques~\citep{dohan2022languagemodelcascades,jitkrittum2023does,yue2024largelanguagemodelcascades,narasimhan2024faster} also utilize multiple language models and decide on the model to used based on certain deferral rule. The switching step in our algorithm can be viewed as imposing a deferral rule based on small, target, and reward models. Additionally, our deferral rule operates at the step-level while previous work usually focuses on individual tokens or the entire generation.

\paragraph{Comparison with concurrent work \cite{geuter2025guidedspeculativeinferenceefficient}.}

The concurrent work of \citet{geuter2025guidedspeculativeinferenceefficient} proposes Guided Speculative Inference (GSI), which also considers soft draft verification by sampling from an exponentially tilted distribution (i.e., soft best-of-$n$). Despite being concurrent, our work provides the following additional novelties compared to \citet{geuter2025guidedspeculativeinferenceefficient}: (1) We propose a novel dynamic switching scheme to improve latency-accuracy tradeoff. Similar to prior work, \citet{geuter2025guidedspeculativeinferenceefficient} starts by generating steps with the draft model, while \specalign starts with the target model and switches to the draft model only for high-reward/easier steps, this has resulted in favorable gains as demonstrated in our experimental section.
(2) Our paper also makes stronger theoretical contributions. In particular, we provide convergence guarantees on the multi-step version of our algorithm, as well as analysis on cases with imperfect reward signals.


We note here that while a newer version of 
\citet{geuter2025guidedspeculativeinferenceefficient} included a comparison of GSI with \specalign, the comparison is slightly inaccurate since they used math-instruct variants of the Qwen models while we used the base instruct models which are not trained for math reasoning. 
In \Cref{tab:specs_gsi}, we include the accuracy of \specalign in the experiment setting of GSI (using Qwen math-instruct models to solve MATH500 and OlympiadBench with $n=4,16$)
along with the reported numbers for GSI in \Cref{tab:specs_gsi}.  We can see that \specalign achieves better accuracy than GSI\footnote{\Cref{tab:specs_gsi} does not report latency numbers, since the work of \citet{geuter2025guidedspeculativeinferenceefficient} only reports the per-step accuracy of GSI averaged across a set of $5$ benchmarks which do not overlap with the benchmarks we evaluate on.}.

\begin{table}[h!]
\caption{Comparison between \specalign (with Qwen2.5-Math-Instruct models) and GSI \cite{geuter2025guidedspeculativeinferenceefficient}}
\centering
\resizebox{0.5\linewidth}{!}{
\begin{tabular}{lc|cc}
\toprule
\textbf{Dataset} & \textbf{Beam Width} ($n$) & \textbf{SPECS} (Acc) & \textbf{GSI} (Acc) \\
\cmidrule(lr){1-4}
MATH500 & 4 & \textbf{82.4}\% & 81.2\% \\
MATH500 & 16 & \textbf{85.7}\% & 82.2\% \\
\midrule
OlympiadBench & 4 & \textbf{41.7}\% & 39.7\% \\
OlympiadBench & 16 & \textbf{44.2}\% & 40.8\% \\
\bottomrule
\end{tabular}
}
\label{tab:specs_gsi}
\end{table}
\section{Theoretical guarantees for \specalign}
\label{app:theory}

In this section, we describe the theoretical guarantees for \specalign. The nature of result we will establish is that, under certain assumptions, as the beam-width increases, the distribution of output of \specalign converges to the optimal policy of the reward maximization problem with KL regularization toward the target model. We first introduce some relevant notation.

\paragraph{Notation.} Denote the space of prompts as $\mathcal{X}$, the space of responses as $\mathcal{Y}$, the draft model $\pidraft : \mathcal{X} \to \Delta_{\mathcal{Y}}$ and the target model $\pitarget : \mathcal{X} \to \Delta_{\mathcal{Y}}$. Prompts and responses are composed of sequences of tokens, and responses are assumed to be complete: ending in a special ``End-Of-Sentence'' ($\texttt{EOS}$) token. We will assume that there exists a target distribution over prompts, denoted $\rho$, and the notation $\mathbb{E}_{\rho,\pi} [\cdot]$ is shorthand for $\mathbb{E}_{X \sim \rho, Y \sim \pi(\cdot|X)} [\cdot]$, where the random variable $Y$ corresponds to full responses drawn from $\pi$ under the prompt $X$. The objective is to find a policy which maximizes the expected reward, while staying close to the target model (as induced by a KL regularization penalty toward $\pitarget$),
\begin{align}
    \mathcal{L}_\beta ( \pi) = \mathbb{E}_{\rho,\pi} [ r(y|x)] - \frac{1}{\beta}\mathbb{E}_{\rho} [ \kl (\pi (\cdot|x) \| \pitarget (\cdot|x))]
\end{align}
For two models $\pi_1$ and $\pi_2$, we use the notation $\kl (\pi_1 \| \pi_2) = \mathbb{E}_{\rho} [ \kl (\pi_1 (\cdot|x) \| \pi_2 (\cdot|x))]$.

Prior to analyzing the algorithm, we begin with a well-known connection between the KL regularized reward maximization objective and the KL divergence between the learnt policy and the optimal aligned policy \cite{korbak2022rl,korbak2022reinforcement,yang2024asymptotics}. We include the proof for the sake of completeness.

\begin{lemma}[Equivalence of KL minimization and KL-regularized reward maximization] \label{lemma:equivalence} For any policy $\pi$,
\begin{align}
    \mathcal{L}_\beta ( \pi^\star_\beta ) - \mathcal{L}_\beta ( \pi) = \kl ( \pi \| \pi^\star_\beta )
\end{align}
where $\pi^\star_\beta$ is the maximizer of $\mathcal{L}_\beta ( \pi)$.
\end{lemma}
\begin{proof}
By \cite{yang2024asymptotics}, the optimal solution to \cref{eq:kl-reg-reward-max}, $\pi^\star_\beta$ takes the structural form,
\begin{align*}
    \pi^\star_\beta (y|x) = \frac{\pitarget (y|x) e^{\beta r(y|x)}}{Z(x)}
\end{align*}
where $Z(x) = \mathbb{E}_{y \sim \pitarget (y|x)} [ e^{\beta r(y|x)}]$.
Thus,
\begin{align*}
    \mathcal{L}_\beta ( \pi) &= \mathbb{E}_{\rho,\pi} [ r(y|x)] - \frac{1}{\beta} \mathbb{E}_{\rho} [ \mathbb{E}_{y \sim \pi (\cdot|x)} [ \log (\pi (y|x) / \pitarget (y|x))]] \\
    &= - \frac{1}{\beta} \mathbb{E}_{\rho} \big[ \mathbb{E}_{y \sim \pi (\cdot|x)} \big[ \log (\pi (y|x) / (\pitarget (y|x)e^{\beta (r(y|x))}) )\big] \big] \\
    &= - \frac{1}{\beta} \mathbb{E}_\rho [\log (Z(x))] - \frac{1}{\beta} \mathbb{E}_{\rho} \big[ \mathbb{E}_{y \sim \pi (\cdot|x)} \big[ \log \big( \pi(y|x) / \pi^\star_\beta (y|x) \big) \big] \big] \\
    &= \mathcal{L}_\beta (\pi^\star_\beta) - \frac{1}{\beta} \mathbb{E}_{\rho} \big[ \kl (\pi (\cdot|x) \| \pi^\star_\beta (\cdot|x)) \big]
\end{align*}
Rearranging completes the proof.
\end{proof}
Therefore, up to an additive error term that does not depend on $\pi$, the loss $\mathcal{L}_\beta (\pi)$ measures the KL divergence between the distribution over responses induced by $\pi$ and $\pi^\star_\beta$ under prompts generated from $\rho$.

The purpose of introducing this result is to motivate the reason for choosing the error measure we study throughout the paper as the \textit{reverse} KL divergence, rather than the forward KL divergence. The former is closely connected with the KL regularized reward maximization objective (it measures the suboptimality gap), while the latter is not so clearly connected. Indeed, these objectives penalize the learnt policy differently: the latter more strongly penalizes policies for placing too high mass on certain responses while the former more strongly penalizes policies which place too low mass on responses.

\subsection{Warm-up: Analysis of \specalign in the Infinite-block length regime} \label{app:warm-up}

In this section we will build up our theoretical understanding of \specalign by considering the infinite block-length setting. In this setting, the drafts generated by the draft model are assumed to be full responses. Here, the algorithm simply selects one candidate among a number of responses using the \subsample subroutine.

\begin{definition}[Infinite block-length regime] \label{def:one-block}
The infinite block-length regime corresponds to taking $\gamma \to \infty$. Operationally, the draft/target model generates full responses and the process reward model collapses to an outcome reward model.
\end{definition}


In this section we will establish a convergence rate of $O(1/n^2)$ for \specalign in the infinite block-length setting. Increasing $n$ scales the amount of (parallel) test-time compute used by the algorithm.

\begin{theorem}[Guarantee for \specalign in the infinite-block length regime] \label{theorem:1-restate}
Suppose $n \ge 3$ and assume that the reward function lies in the range $[0,R]$ pointwise. Then, the policy $\pi_{\specalign}$ returned by \specalign \ in the infinite block-length setting (\Cref{def:one-block}), 
, satisfies,
\begin{align} \label{eq:001}
    \kl (\pi_{\specalign} \| \pi_{\beta}^\star ) \le \widetilde{\mathcal{O}}_n \left( \Cdt^2 \frac{e^{2\beta R}}{n^{3/2}} \right)
\end{align}
where, the ``sequence-level coverage coefficient'' $\Cdt$ is,
\begin{align} \label{eq:Cdt}
    \Cdt = \left\| \frac{\pitarget (\cdot|x)}{\pidraft (\cdot|x)} \right\|_\infty \left\| \frac{\pidraft (\cdot|x)}{\pitarget (\cdot|x)} \right\|_\infty.
\end{align}
\end{theorem}

\begin{remark}
In our analysis we establish a stronger result, where the LHS of Eq.~(\ref{eq:001}) is replaced by $\log(1 + D_{\chi^2} ( \pi_{\specalign} \| \pi^\star_\beta))$.
\end{remark}


The proof of this result follows as a consequence of \Cref{corr:specsampler} which we will introduce later. In the infinite block-length regime, \specalign runs \Cref{alg:subsample} once with the generator model $\pigen \gets \pitarget$. In the next section, we will prove a guarantee on the distribution over responses / steps induced \subsample routine  we consider in \Cref{alg:subsample} as an instantiation of sequential monte carlo (SMC) \cite{rubin1987calculation,naesseth2019elements}, also known by the name particle filtering in the literature \cite{gordon1993novel}. In the next section, we will introduce SMC and prove a distributional closeness result for the resulting policy. In turn, this will imply the proof of \Cref{theorem:1-restate} and provide a starting point for proving \Cref{theorem:specalign-block}.

\subsection{Sequential Monte Carlo: Sampling by approximating the partition function}

Given $N$ responses drawn from the generator model $\pigen(\cdot|x)$, the partition function $Z(x)$ in \cref{eq:opt-redefine} can be approximated by a discrete summation,
\begin{equation*}
   Z(x) = \mathbb{E}_{Y \sim \pigen (\cdot|x)} [e^{S(Y|x)}] \approx \frac{1}{N} \sum_{i=1}^N e^{S(Y_i|x)} = \widehat{Z} (x).
\end{equation*}
SMC uses this approximation to sample from the target distribution, $\pigen (\cdot|x) e^{S(\cdot|x)} / Z(x)$.

\begin{definition}[Sequential Monte Carlo (SMC)]
\label{def:SMC}
Given a prompt $x$, SMC generates a set of draft responses $\Ydraft = \{ Y_i \}_{i=1}^N$ i.i.d. from $\pigen(\cdot|x)$. Then an index $i$ is sampled with probability proportional to $e^{S(Y_i|x)}$ and $Y_i$ is returned. Denoting $Y^\star$ as the output of the procedure,
\begin{align*}
    \Pr (Y^\star = Y_i) &= \frac{e^{S(Y_i|x)}}{Z_{\Ydraft} (x)}, \\
    \text{where, } Z_{\Ydraft} (x) &= \sum_{j \in [n]} e^{S(Y_j|x)}.
\end{align*}
The policy, $\pipartapx{N}$, is the induced distribution over responses, averaged over the randomness of sampling $\Ydraft$.
\end{definition}

 Let us first introduce some notation to make the presentation simpler.

\begin{definition}[Exponentiated score] \label{def:score} For the sake of brevity, define the exponentiated score,
\begin{align*}
    \phi_\beta (y|x) = \frac{\pitarget (y|x)}{\pigen (y|x)} e^{\beta r (y|x)}.
\end{align*}
Furthermore, define $\phi_{\max} = \max_{x \in \mathcal{X},y \in \mathcal{Y}} \phi_\beta (y|x)$ and $\phi_{\min} = \min_{x \in \mathcal{X},y \in \mathcal{Y}} \phi_\beta (y|x)$. Note that $\phi_\beta$ is simply the exponentiated version of the score $S(y|x)$ defined in \cref{eq:score1}.
\end{definition}

To aid our analysis, we will establish guarantees for \specalign which bounds the KL divergence of with the optimal policy, when the number of drafts generated (i.e., the beam-width) is randomized and distributed according to a Poisson random variable with mean $n$. We will establish upper bounds on $\bbE[ \kl (\pipartapx{N} \| \pi_{\beta}^\star ) ]$ where $N \sim \Poi (n)|_{>0}$. Note that since $\Pr (N = n) \gtrsim 1/n$ for $N \sim \Poi (n)|_{>0}$ so, these results can be translated into an upper bound on $\kl (\pipartapx{n} \| \pi_{\beta}^\star )$ for any fixed value of $n$, by multiplying by an additional $\sqrt{n}$ factor.

\begin{theorem}
\label{theorem:soft-bon}
Fix any $n \ge 3$. In the infinite block-length setting, the policy $\pipartapx{N}$ induced by SMC (\Cref{def:SMC}). satisfies,
\begin{align*}
    \bbE_{N \sim \Poi (n)|_{>0}} [ \kl (\pipartapx{N} \| \pi_{\beta}^\star ) ] &\le \mathcal{O} \left( \frac{\log^2(n)}{n^2} \left( \frac{\phi_{\max}}{\phi_{\min}} \right)^2 \right).
\end{align*}
$\phi_{\max}$ and $\phi_{\min}$ are defined in \Cref{def:score}.
\end{theorem}

This result presents a guarantee for SMC where the number of beams generated is drawn from a Poisson distribution. Poisson random variables are tightly concentrated around their mean and the above result implies as a corollary that for any $n \ge 3$,
\begin{align*}
    \min_{N \le n} \kl (\pipartapx{N} \| \pi_{\beta}^\star ) &\le \mathcal{O} \left( \frac{\log^2(n)}{n^2} \left( \frac{\phi_{\max}}{\phi_{\min}} \right)^2 \right).
\end{align*}
Similarly, as mentioned previously, note that for $N \sim \Poi (n)|_{>0}$, $\Pr (N = n) \gtrsim 1/\sqrt{n}$. Therefore, by an application of Markov's inequality, we also arrive at the following bound for any fixed value of $n \ge 3$,
\begin{align}
    \kl (\pipartapx{n} \| \pi_{\beta}^\star ) &\le \mathcal{O} \left( \frac{\log^2(n)}{n^{3/2}} \left( \frac{\phi_{\max}}{\phi_{\min}} \right)^2 \right).
\end{align}
This establishes an upper bound for SMC which applies for any fixed beam width.

Since it is more involved, we defer the proof of \Cref{theorem:soft-bon} to \Cref{sec:proof-sampling-partition}. We have the following two corollaries of \Cref{theorem:soft-bon} by instantiating $\pigen$ differently.

\begin{corollary}[Target sampling] \label{corr:targetsampler}
Suppose $n \ge 3$ and that drafts are generated from $\pitarget$ (i.e., we choose $\pigen \gets \pitarget$). Assuming rewards lie in the range $[0,R]$, in the infinite block-length setting (\Cref{def:one-block}),
\begin{align*}
    \bbE_{N \sim \Poi (n)|_{>0}} [ \kl (\pipartapx{N} \| \pi_{\beta}^\star ) ] \le \widetilde{\mathcal{O}}_n \left( \frac{e^{2\beta R}}{n^2} \right)
\end{align*}
By an application of Markov's inequality, we have that for any fixed $n \ge 3$,
\begin{align*}
    \kl (\pipartapx{n} \| \pi_{\beta}^\star ) \le \widetilde{\mathcal{O}}_n \left( \frac{e^{2\beta R}}{n^2} \right)
\end{align*}
\end{corollary}

Note that \specalign in the infinite-block length regime is precisely the same as $\pipartapx{N}$ where the generator $\pigen \gets \pitarget$. This implies the statement of \Cref{theorem:1-restate}.

\begin{corollary}[Speculative sampling] \label{corr:specsampler}
Suppose $n \ge 3$ and that drafts are generated from $\pidraft$ (i.e., we choose $\pigen \gets \pidraft$). Then, in the infinite block-length setting (\Cref{def:one-block}), assuming rewards lie in the range $[0,R]$,
\begin{align*}
    \bbE_{N \sim \Poi(n)|_{>0}} [ \kl (\pipartapx{N} \| \pi_{\beta}^\star )] \le \widetilde{\mathcal{O}}_n \left( \Cdt^2 \frac{e^{2\beta R}}{n^2} \right)
\end{align*}
where $\Cdt = \| \pitarget (\cdot|x) / \pidraft (\cdot|x) \|_\infty \| \pidraft (\cdot|x) / \pitarget (\cdot|x) \|_\infty$, as defined in \Cref{eq:Cdt}. By an application of Markov's inequality, we have that for any fixed $n \ge 3$,
\begin{align*}
    \kl (\pipartapx{n} \| \pi_{\beta}^\star ) \le \widetilde{\mathcal{O}}_n \left( \Cdt^2 \frac{e^{2\beta R}}{n^{3/2}} \right)
\end{align*}
\end{corollary}

\begin{remark} \label{remark:1}
In all three prior results, \Cref{theorem:soft-bon,corr:targetsampler,corr:specsampler}, the proof will argue the stronger upper bound which replaces the KL divergence in the LHS by the strictly larger quantity, $\log (1 + D_{\chi^2} (\pipartapx{N} \| \pi_{\beta}^\star ))$.
\end{remark}

When the target and draft models are identical, then $\Cdt = 1$. We also show that the rate of decay in \Cref{corr:specsampler} is essentially tight when the target and draft models are identical.

\begin{theorem} \label{theorem:lb-softbon} There exists a target model $\pitarget$ and reward model over a $2$-ary response space $\mathcal{Y} = \{ y_0,y_1 \}$, with rewards lying in the range $[0,R]$, such that, the policy $\pipartapx{N}$ returned by sampling by SMC (\Cref{def:SMC}) in the infinite block-length setting (\Cref{def:one-block}) with $\pigen \gets \pitarget$,
\begin{align*}
    \bbE_{N \sim \Poi (n)|_{>0}} [\kl (\pipartapx{N} \| \pi_{\beta}^\star )] \ge \Omega_n \left( \frac{e^{2 \beta R}}{n^2} \right),
\end{align*}
assuming that $n \ge e^{\beta R}$.
\end{theorem}

\subsection{Finite-block length: Analysis of \specalign with idealized process feedback}

In this section we extend the previous discussion to the finite-block setting. In this case, \specalign \ (\Cref{alg:special}) is iterative and  the reward model considered in \subsample (\Cref{alg:subsample}) must be implemented with a process reward model (PRM) which is trained to provide block-level feedback. Furthemore, control of draft generation is shared between the draft and target models additionally through dynamic switching.

In the infinite block-length setting, the guarantee of \specalign depended on a strong notion of coverage, $\Cdt = \| \pitarget (\cdot|x) / \pidraft (\cdot|x) \|_\infty \| \pidraft (\cdot|x) / \pitarget (\cdot|x) \|_\infty$, where the distributions are over full responses, a quantity which may scale exponentially with $H$. In the finite block-length setting, our main result, \Cref{theorem:specalign-block}  shows that with access to a PRM in the sense of \Cref{def:prm} allows guarantees that scale polynomially with the horizon $H$, where the dependency on $\Cdt$ is relaxed to the maximum ``per-block'' ratio of densities between the target and draft models.




We will reiterate the definitions of the process reward model we consider. In practice this model is trained and inexact, but for the purpose of a clean theoretical discussion we avoid this challenge and discuss the case where the process reward is exact. The following two definitions instantiate the PRM.

\begin{definition}[Optimal KL-regularized value function: restatement of \Cref{def:value}] \label{def:value-repeat}
For a policy $\pi$, and any prompt $x \in \mathcal{X}$ and partial trace $y_{\le t}$, the optimal KL-regularized value of a policy is defined as,
\begin{align*}
    V^{\pi}_\beta (x, y_{\le t}) = \frac{1}{\beta} \log \left( \mathbb{E}_{Y \sim \pi ( \cdot|x,y_{\le t})} \left[ e^{\beta r (x, Y)} \right] \right)
\end{align*}
where $Y$ is a full completion of the partial response $y_{\le t}$. Note that for any full response $y \in \mathcal{Y}$, $V^{\pi}_\beta (x, y) = r(x,y)$.
\end{definition}

\begin{definition}[Idealized Process Reward Model: restatement of \Cref{def:prm}] \label{def:prm-repeat} Let $\rprm^\pi$ denote the PRM corresponding to the optimal KL-regularized advantage function with respect to some policy $\pi$. Formally, given any prefix of tokens / blocks $y_{\le t-1}$ of length $t-1$ and a subsequent token / block $y_t$,
\begin{align*}
	\rprm^\pi (y_t |x, y_{\le t-1}) &= V^{\pi}_\beta (x, y_{\le t}) - V^{\pi}_\beta (x, y_{\le t-1})
\end{align*}
where, the KL-regularized value function $V^{\pi}_\beta$ is defined in \Cref{def:value}. The cumulative advantage till block $t$ is denoted, $\rprm^\pi (y_{\le t} |x) = \sum_{t'=1}^t \rprm^\pi (y_{t'} |x, y_{\le t'-1})$. The idealized PRM we consider is $\rprmtarget = \rprm^\pi$ for $\pi=\pitarget$.
\end{definition}

\medskip
Our first result will be to establish that the optimal aligned model can be decomposed into a sequence of models corresponding to solving ``one-block'' regularized reward maximization problems over the reward function induced by $\rprmtarget$.

\begin{lemma} \label{lemma:product}
The optimal aligned model over full responses can be written as a product of conditional per-block models,
\begin{align} \label{eq:decomp}
    \pi^\star_\beta ( y_1,\cdots,y_t | x ) = \prod_{t=1}^H \pi^\star_\beta (y_t | x, y_{\le t-1} ),
\end{align}
where, we define,
\begin{align*}
\pi^\star_\beta (y_t | x, y_{\le t-1} ) \propto \pitarget(y_t | x, y_{\le t-1}) e^{\beta [ \rprmtarget (y_t | x, y_{\le t-1})]}
\end{align*}
Is the solution to a one-block KL-regularized reward maximization problem. Namely, for every prompt $x$ and partial trace $y_{\le t-1}$, $\pi^\star_\beta (\cdot | x,y_{\le t-1})$ is supported over one-block completions and is the solution to the problem,
\begin{align*}
    \min \left\{ \mathbb{E}_{Y_t \sim \pi (\cdot | x, y_{\le t-1})} \left[ \rprmtarget (y_t|x,y_{t-1}) \right] - \frac{1}{\beta} \kl (\pi(\cdot | x,y_{\le t-1}), \pitarget (\cdot|x,y_{\le t-1}))  \right\}
\end{align*}
where $\pi : \mathcal{X} \times \mathcal{Y}_{\le t-1} \to \Delta_{\mathcal{Y}_t}$.
\end{lemma}
\begin{proof}
By definition, the optimal aligned policy can be written down as,
\begin{align*}
\pi^\star_\beta (y|x)
&= \frac{\pitarget(y|x) e^{\beta r(y|x)}}{\mathbb{E}_{Y \sim \pitarget(\cdot|x)} [e^{\beta r(Y|x)}]} \\
&= \pitarget(y|x) e^{\beta [ V^{\pitarget}_\beta (x, y) - V^{\pitarget}_\beta (x,\emptyset)] } \\
&= \pitarget(y|x) \prod_{t'=1}^H e^{\beta [ V^{\pitarget}_\beta (x, y_{\le t'}) - V^{\pitarget}_\beta (x,y_{\le t'-1})] } \\
&= \prod_{t=1}^H \frac{\pitarget(y_t | x, y_{\le t-1}) e^{\beta [ V^{\pitarget}_\beta (x, y_{\le t})}}{\mathbb{E}_{Y_t \sim \pitarget (\cdot|x,y_{\le t-1})} [ e^{\beta V^{\pitarget}_\beta (x, [y_{\le t-1},Y_t] )} ] } \\
&= \prod_{t=1}^H \frac{\pitarget(y_t | x, y_{\le t-1}) e^{\beta [ \rprmtarget (y_t | x, y_{\le t-1})}}{\mathbb{E}_{Y_t \sim \pitarget (\cdot|x,y_{\le t-1})} [ e^{\beta \rprmtarget (Y_t | x, y_{\le t-1} )} ] }.
\end{align*}
\end{proof}

Next, we show that policies which can guarantee that they realize a one-block distribution which is close to $\pi^\star_\beta (\cdot|x,y_{\le t-1})$ for every $t$ correspond to policies which is close to $\pi^\star_\beta$ over full responses in the aggregate.

\begin{lemma}[Local-to-global guarantee] \label{lemma:cascade}
Consider a policy with next-block distribution given by $\pi (\cdot | x, y_{t-1})$ which satisfies the following guarantee: for every prompt $x$ and prefix $y_{t-1}$,
\begin{align*}
D_{\chi^2} ( \pi (\cdot | x, y_{\le t-1} ) , \pi^\star_\beta (\cdot | x, y_{\le t-1} ) ) \le \varepsilon.
\end{align*}
where the induced distributions $\pi(\cdot|x,y_{\le t-1})$ and $\pi^\star_\beta (\cdot|x,y_{\le t-1})$ are one-block completions of the partial response $(x,y_{\le t-1})$. Then, we have that,
\begin{align*}
1 + D_{\chi^2} ( \pi , \pi^\star_\beta ) \le (1 + \varepsilon)^H
\end{align*}
\end{lemma}
\begin{proof}
By definition
\begin{align*}
&1 + D_{\chi^2} ( \pi \| \pi^\star_\beta ) \\
&= \mathbb{E}_\rho \left[ \mathbb{E}_{Y \sim \pi^\star_\beta (\cdot|x)} \left[ \left( \frac{\pi (Y|x)}{\pi^\star_\beta (Y|x)} \right)^2 \right] \right] \\
&= \mathbb{E}_\rho \left[ \mathbb{E}_{Y \sim \pi^\star_\beta (\cdot|x)} \left[ \prod_{t=1}^H \left( \frac{\pi (Y_t|x,Y_{\le t-1})}{\pi^\star_\beta (Y_t|x, Y_{\le t-1})} \right)^2 \right] \right] \\
&= \mathbb{E}_\rho \left[ \mathbb{E}_{Y_{\le H-1} \sim \pi^\star_\beta (\cdot|x)} \left[ [1 + D_{\chi^2} ( \pi (\cdot |x,Y_{\le H-1}) \| \pi^\star_\beta (\cdot |x, Y_{\le H-1}) ) ] \cdot \prod_{t=1}^{H-1} \left( \frac{\pi (Y_t|x,Y_{\le t-1})}{\pi^\star_\beta (Y_t|x, Y_{\le t-1})} \right)^2 \right] \right] \\
&\overset{(a)}{\le} (1+\varepsilon) \mathbb{E}_\rho \left[ \mathbb{E}_{Y_{\le H-1} \sim \pi^\star_\beta (\cdot|x)} \left[ \prod_{t=1}^{H-1} \left( \frac{\pi (Y_t|x,Y_{\le t-1})}{\pi^\star_\beta (Y_t|x, Y_{\le t-1})} \right)^2 \right] \right],
\end{align*}
where $(a)$ follows from the local assumption on $\pi$. By recursing on the RHS, we arrive at the statement of the lemma.
\end{proof}

Combining this recursion with our guarantees for \specalign in the infinite block-length setting, we prove the main result \Cref{theorem:specalign-block} next.

\paragraph{Proof of \Cref{theorem:specalign-block}.}

Consider some partial trace $y_{\le t-1}$ and let us consider the process of generating the $t^{\text{th}}$ block. Let $Z_t \in \{ 0,1 \}$ denote the random variable which is a measurable function of $y_{\le t-1}$ which determines whether $\pidraft$ generates the drafts or $\pitarget$ generates drafts in iteration $t$.

By the observation in \Cref{remark:1} for the infinite block-length instantiation of the \specalign \ policy $\pi_{\specalign}$, the following local guarantee is also true when instantiated with $r = \rprmtarget$ and used to generate a single block: given any prompt $x$ and partial trace $y_{\le t-1}$,
\begin{align*}
    D_{\chi^2} ( \pipartapx{n} (\cdot|x,y_{\le t-1}) |_{\{Z_t = 0\}} \| \pi_{\beta}^\star (\cdot|x,y_{\le t-1})) &\le \exp \left( \widetilde{\mathcal{O}}_n \left( \frac{e^{2 \beta R}}{n^{3/2}} \right) \right) - 1 \\
    D_{\chi^2} ( \pipartapx{n} (\cdot|x,y_{\le t-1}) |_{\{Z_t = 1\}} \| \pi_{\beta}^\star (\cdot|x,y_{\le t-1})) &\le \exp \left( \widetilde{\mathcal{O}}_n \left( \Cdtstep^2 \frac{e^{2 \beta R}}{n^{3/2}} \right) \right) - 1
\end{align*}
Note in particular that the second guarantee only depends on $\Cdtstep$ by virtue of the fact that we only generate a single block and not entire responses ($\pi_{\specalign} (\cdot|x,y_{t-1})$ is a distribution over single blocks). By the convexity of the $\chi^2$-divergence, we have that $\pi_{\specalign} (\cdot|x,y_{\le t-1})$, which is a mixture between $\pipartapx{n} (\cdot|x,y_{\le t-1})$ with $\pigen \gets \pidraft$ (i.e., $Z_t=0$) and $\pipartapx{n} (\cdot|x,y_{\le t-1})$ with $\pigen \gets \pitarget$ (i.e., $Z_t=1$) satisfies,
\begin{align*}
    D_{\chi^2} ( \pi_{\specalign} (\cdot|x,y_{\le t-1}) \| \pi_{\beta}^\star (\cdot|x,y_{\le t-1})) &\le \exp \left( \widetilde{\mathcal{O}}_n \left( \Cdtstep^2 \frac{e^{2 \beta R}}{n^{3/2}} \right) \right) - 1
\end{align*}
The proof finishes by combining these guarantees across $t$ via \Cref{lemma:cascade} to give the inequality,
\begin{align*}
    1 + D_{\chi^2} ( \pi_{\specalign} \| \pi_{\beta}^\star) \le \exp \left( \widetilde{\mathcal{O}}_n \left( H \cdot \Cdtstep^2 \frac{e^{2 \beta R}}{n^{3/2}} \right) \right)
\end{align*}
Using the inequality, $\kl ( \pi_{\specalign} \| \pi_{\beta}^\star) \le \log (1 + D_{\chi^2} ( \pi_{\specalign} \| \pi_{\beta}^\star))$ completes the proof.

\section{Robustness of \specalign to misspecified PRMs} \label{sec:apx-prm}

Our guarantees for \specalign in the previous section hinge on the assumption that the PRM is idealized, and equal to the optimal KL-regularized advantage function. However, PRMs in practice can often be noisy. In this section, we consider a model of error which still enables benefiting from process feedback at test-time. We weaken the assumption on the PRM by instead assuming that it satisfies a pointwise error guarantee with respect to the optimal KL-regularized advantage function (i.e., the idealized PRM). 

\begin{assumption} \label{ass:prm-apx}
Recall the definition of $\rprmtarget$ as defined in \Cref{def:prm}. Suppose the learner's PRM $\rprm$ satisfies the conditioned,
\begin{equation*}
    \| \rprm - \rprmtarget \|_\infty \le \varepsilon_{\texttt{PRM}}.
\end{equation*}
where $\varepsilon_{\texttt{PRM}} > 0$ is a misspecification parameter.
\end{assumption}

When \specalign is implemented with this misspecified PRM, we show that the KL divergence of the resulting policy to the optimal policy $\pi_\beta^\star$ degrades gracefully with $\varepsilon_{\texttt{PRM}}$.

\begin{theorem} \label{theorem:prm-apx}
Assume $n \ge 3$ and suppose \specalign \ is implemented with a PRM which satisfies \Cref{ass:prm-apx}, and with beam-width $n$. Then, for reasoning problems over $H$ blocks, in the finite-block length setting, the policy $\pi_{\specalign}$ returned by \specalign \ satisfies,
\begin{align*}
    \kl (\pi_{\specalign} , \pi^\star_\beta) \le \widetilde{\mathcal{O}}_n \left( H \cdot \Cdtstep^2 \frac{e^{2 \beta R}}{n^{3/2}} + \beta H \varepsilon_{\texttt{PRM}} \right)
\end{align*}
Here, we assume that the PRM reward range is $[0,R]$.
\end{theorem}
Thus, even when the PRM departs from the idealized form we assume earlier in \Cref{ass:prm-apx}, the algorithm still admits convergence guarantees toward the optimal policy under the \textit{ground-truth reward function}, $\pi^\star_\beta$.

\paragraph{Proof of \Cref{theorem:prm-apx}.}
With an arbitrary choice of PRM $\rprm$, it is a simple modification of the proof of \Cref{theorem:soft-bon} to see that,
\begin{align*}
    \kl ( \pi_{\specalign} \| \pi_{\beta,\prm}^\star) \le \widetilde{\mathcal{O}}_n \left( H \cdot \Cdtstep^2 \frac{e^{2 \beta R}}{n^{3/2}} \right)
\end{align*}
Where, similar to \cref{eq:decomp},
\begin{align*}
    \pi_{\beta,\prm}^\star (y | x) \triangleq \prod_{t=1}^H \pi_{\beta,\prm}^\star (y_t | x, y_{\le t-1})
\end{align*}
Where, we define,
\begin{align*}
\pi^\star_{\beta,\prm} (y_t | x, y_{\le t-1} ) \propto \pitarget(y_t | x, y_{\le t-1}) e^{\beta [ \rprm (y_t | x, y_{\le t-1})]}
\end{align*}
The statement of theorem $2$ will follow from the following inequality,
\begin{align}
    \kl ( \pi_{\specalign} \| \pi_{\beta}^\star) \le \kl ( \pi_{\specalign} \| \pi_{\beta,\prm}^\star) + \log \left( \left\| \frac{\pi_{\beta}^\star}{\pi_{\beta,\prm}^\star} \right\|_\infty \right) \label{eq:99912}
\end{align}
What remains is to bound the worst-case density ratio between $\pi^\star_\beta$ and $\pi^\star_{\beta,\prm}$. Noting from \cref{eq:decomp},
\begin{align*}
    &\frac{\pi_{\beta}^\star (y_t | x, y_{\le t-1})}{\pi_{\beta,\prm}^\star (y_t | x, y_{\le t-1})} \\
    &= \frac{e^{\beta [ \rprmtarget (y_t | x, y_{\le t-1}) ]}}{e^{\beta [ \rprm (y_t | x, y_{\le t-1})]}} \cdot \frac{\sum_{y'} \pitarget(y' | x, y_{\le t-1}) e^{\beta [ \rprm (y_t | x, y_{\le t-1})]}}{\sum_{y'} \pitarget(y' | x, y_{\le t-1}) e^{\beta [ \rprmtarget (y_t | x, y_{\le t-1})]}} \\
    &\le e^{2\beta \varepsilon_{\texttt{PRM}}}
\end{align*}
Where the last inequality uses the fact that $\rprm$ satisfies \Cref{ass:prm-apx}. Combining with \cref{eq:99912} completes the proof.

\section{Analysis of SMC: Proof of \Cref{theorem:soft-bon,theorem:lb-softbon}} \label{sec:proof-sampling-partition}

We begin with the proof of \Cref{theorem:soft-bon}; the last part of this section will be dedicated toward proving \Cref{theorem:lb-softbon}. Prior to diving into the proof, we will first introduce some additional notation. Let $N_y$ denote the number of occurrences of a particular $y \in \mathcal{Y}$ in the collection of $N \sim \textsf{Poi}(n)_{>0}$ responses, $\Ydraft$, drawn i.i.d. from $\pigen(\cdot|x)$. Recall the definition of the score function $\phi_\beta (\cdot|\cdot)$ as defined in \Cref{def:score}. Furthermore, define,
\begin{align*}
\Phi_\beta (s) &= \mathbb{E}_{Y \sim \pigen(\cdot|x)} \left[e^{-s\,\phi_\beta (Y|x)}\right]
\end{align*}
as the Laplace transform of $\phi_\beta (Y|x)$ for $Y \sim \pigen(\cdot|x)$. Define,
\begin{align} \label{eq:pipartapxN}
    \pipartapx{N} (y|x) &= \mathbb{E} \left[ \frac{N_y \phi_\beta (y|x)}{Z_{\Ydraft,\beta}(x)} \middle| N \right]
\end{align}
Furthermore, let us define,
\begin{align}
    \barpipartapx{n} (y|x) &= \bbE_{N \sim \Poi(n)|_{>0}} [\pipartapx{N} (y|x)],
\end{align}
where $N_0 \sim \textsf{Poi} (n)$.
Note that with our definition of scores (\Cref{def:score}), we have,
\begin{align*}
    \pi_{\beta}^\star (y|x) = \frac{\pitarget (y|x) e^{\beta r (y|x)}}{\mathbb{E}_{Y \sim \pitarget(\cdot|x)} [ e^{\beta r(Y|x)}]} = \frac{\pigen(y|x) \phi_\beta (y|x)}{\mathbb{E}_{Y \sim \pigen(\cdot|x)} [ \phi_\beta (Y|x)] } = \frac{\pigen(y|x) \phi_\beta (y|x)}{-\Phi_\beta' (0)}.
\end{align*}
Our objective is to upper bound,
\begin{align*}
    \mathbb{E}_{N \sim \textsf{Poi} (n) |_{>0}} \left[ D_{\chi^2} (\pipartapx{N} \| \pi_{\beta}^\star ) \right] &= \sum_{y \in \mathcal{Y}} \mathbb{E}_{N \sim \textsf{Poi} (n) |_{>0}} \left[ \frac{(\pipartapx{N} (y|x))^2}{\pi_{\beta}^\star (y|x)} \right]
\end{align*}

Our first result will be to provide a formula for the probability that a response $y$ is sampled by $\barpipartapx{n} = \bbE [ \pipartapx{N}]$.

\begin{lemma} \label{lemma:integral-0}
$\barpipartapx{n}$ can be written as the following integral,
\begin{align*}
\barpipartapx{n} (y|x) = \frac{n \pigen(y|x) \phi_\beta (y|x)}{1 - e^{-n}} \int_0^\infty e^{n(\Phi_\beta (s)-1)} \cdot e^{-s \phi_\beta (y|x)} \dd s.
\end{align*}
\end{lemma}
\begin{proof}
Our first step we will be to arrive at an integral representation of $1/Z_{\Ydraft,\beta}(x)$. For any $A > 0$, $A^{-1} = \int_0^\infty e^{-s A} \dd s$. With the choice $A = Z_{\Ydraft,\beta}(x)$, by Fubini's theorem,
\begin{align}
\mathbb{E} \left[\frac{\mathbb{I} (N \ge 1) \cdot N_y \phi_\beta (y|x)}{Z_{\Ydraft,\beta}(x)} \right] &= \phi_\beta (y|x) \int_0^\infty \mathbb{E} \left[ N_y e^{-s\,Z_{\Ydraft,\beta}(x)}\, \mathbb{I} (N \ge 1) \right]\,\dd s = \int_0^\infty \Psi_y (s)\,\dd s \label{eq:Psi-0}
\end{align}
where we define,
\begin{align*}
\Psi_y (s) = \phi_\beta (y|x) \mathbb{E} \left[ N_y e^{-s\,Z_{\Ydraft,\beta}(x)}\,\mathbb{I} (N_0 \ge 1) \right] = \phi_\beta (y|x) \mathbb{E} \left[ N_y e^{-s\,Z_{\Ydraft,\beta}(x)} \right],
\end{align*}
which uses the fact that $N_y = 0$ if $N_0 = 0$. The remainder of this proof will be to analyze $\Psi_y (s)$.\\ Define $N_y = \sum_{y' \in \Ydraft} \mathbb{I} (y' = y)$ as the number of times $y$ appears in $\Ydraft$. Then,
\begin{align*}
\Psi_y (s) &= \mathbb{E} \left[ N_y e^{-s\,\sum_{y_0 \in \mathcal{Y}} N_{y_0} \phi_\beta (y_0|x) } \right] \\
 &= \mathbb{E} \left[ N_y e^{-s N_y \phi_\beta (y|x)} \right] \cdot \prod_{y_0 \in \mathcal{Y} \setminus \{ y \}}\mathbb{E} \left[ e^{-s N_{y_0} \phi_\beta (y_0|x) } \right] 
\end{align*}
where the second equation uses independence admitted by the Poisson structure. Note that,
\begin{align*}
\mathbb{E} \left[ e^{-s N_{y_0} \phi_\beta (y_0|x) } \right] &= \sum_{j=0}^\infty e^{-s j \phi_\beta (y_0|x) } \Pr [N_y = j] \\
&= \sum_{j=0}^\infty e^{-s j \phi_\beta (y_0|x) } \cdot \frac{(n \pigen (y_0|x))^j}{j!} e^{-n \pigen (y_0|x)} \\
&= \exp (n \pigen (y_0|x) \big(e^{-s \phi_\beta (y_0|x) } - 1\big))
\end{align*}
A similar calculation results in,
\begin{align*}
\mathbb{E} \left[ N_y e^{-s N_{y} \phi_\beta (y|x) } \right] &= \sum_{j=0}^\infty j e^{-s j \phi_\beta (y|x) } \Pr [N_y = j] \\
&= \sum_{j=1}^\infty e^{-s j \phi_\beta (y|x) } \cdot \frac{(n \pigen(y|x))^j}{(j-1)!} e^{-n \pigen(y|x)} \\
&= n \pigen(y|x) e^{-s \phi_\beta (y|x) } \sum_{j=0}^\infty \cdot \frac{\big( n \pigen(y|x) e^{-s \phi_\beta (y|x) } \big)^j}{j!} e^{-n \pigen(y|x)} \\
&= n \pigen(y|x) e^{-s \phi_\beta (y|x) } \exp (n \pigen(y|x) \big(e^{-s \phi_\beta (y|x) } - 1\big))
\end{align*}
Multiplying everything out, we get,
\begin{align*}
\Psi_y (s) &= n \pigen(y|x) e^{-s \phi_\beta (y|x) } \exp (n \pigen(y|x) \big( e^{-s \phi_\beta (y|x) } - 1 \big)) \cdot \prod_{y_0 \in \mathcal{Y} \setminus \{ y \}} \exp (n \pigen (y_0|x) \big(e^{-s \phi_\beta (y_0|x) } - 1\big)) \\
&= n \pigen(y|x) e^{-s \phi_\beta (y|x)} e^{n (\Phi_\beta (s)-1)}
\end{align*}
Plugging into \cref{eq:Psi} completes the proof.
\end{proof}

\noindent As a sanity check, we can check that the probability calculated in \Cref{lemma:integral-0} integrates to $1$ over $y \in \mathcal{Y}$. Notably,
\begin{align*}
    \sum_{y \in \mathcal{Y}} \barpipartapx{n} (y|x) &= \sum_{y \in \mathcal{Y}} \frac{n \pigen(y|x) \phi_\beta (y|x)}{1 - e^{-n}} \int_0^\infty e^{n(\Phi_\beta (s)-1)} \cdot e^{-s \phi_\beta (y|x)} \dd s \\
    &= \frac{n}{1 - e^{-n}} \int_0^\infty e^{n(\Phi_\beta (s)-1)} \cdot \mathbb{E}_{Y \sim \pigen(\cdot|x)} [ \phi_\beta (Y|x) e^{-s \phi_\beta (Y|x)}] \dd s \\
    &= \frac{- n}{1 - e^{-n}} \int_0^\infty e^{n(\Phi_\beta (s)-1)} \cdot \Phi'_\beta (s) \dd s \\
    &\overset{(a)}{=} \frac{1}{1 - e^{-n}} \int_0^{n} e^{-t} \cdot \dd t \\
    &= 1
\end{align*}
where in $(a)$, $t = n (1-\Phi_\beta (s))$.

Our next result will be to provide a formula for $\mathbb{E} [ (\pipartapx{N} (y|x))^2 ]$.

\begin{lemma} \label{lemma:integral}
$\mathbb{E} [ (\pipartapx{N} (y|x))^2 ]$ has the following integral form,
\begin{align*}
&\mathbb{E} [ (\pipartapx{N} (y|x))^2 ] \\
&= \frac{(\phi_\beta (y|x))^2}{1 - e^{-n}} \left( \int_0^\infty \int_0^\infty \big( n + n^2 \Phi_\beta (s_1) \Phi_\beta (s_2) \big) \left( (\pigen(y|x))^2 e^{- (s_1+s_2) \phi_\beta (y|x)} \right) \right. \\
&\qquad \times \exp (n \big[ \Phi_\beta (s_1) \Phi_\beta (s_2) - 1\big]) \dd s_1 \, \dd s_2 \bigg)
\end{align*}
\end{lemma}
\begin{proof}
For any $A > 0$, note that $A^{-1} = \int_0^\infty e^{- s A} \dd s$. By Fubini's theorem, and letting $N_0 \sim \textsf{Poi} (n)$,
\begin{align}
&\underbrace{\Pr (N_0 \ge 1)}_{= 1 - e^{-n}} \cdot \mathbb{E}_N [ (\pipartapx{N} (y|x))^2 ] \nonumber\\
&= \mathbb{E} \left[ \mathbb{E} \left[\frac{\mathbb{I} (N_0 \ge 1) \cdot N_y \phi_\beta (y|x)}{Z_{\Ydraft,\beta}(x)} \middle| N_0 \right]^2 \right] \nonumber\\
&\overset{(a)}{=} (\phi_\beta (y|x))^2 \cdot \mathbb{E} \left[ \frac{\mathbb{I} (N_0 \ge 1) \cdot N_y \widetilde{N}_y }{Z_{\Ydraft,\beta}(x) \cdot \widetilde{Z}_{\Ydraft,\beta}(x)} \right] \nonumber\\
&= (\phi_\beta (y|x))^2 \int_0^\infty \int_0^\infty \mathbb{E} \left[ N_y \widetilde{N}_y \cdot e^{- s_1 \, Z_{\Ydraft,\beta}(x) - s_2 \, \widetilde{Z}_{\Ydraft,\beta}(x)}\, \mathbb{I} (N_0 \ge 1) \right]\,\dd s_1 \, \dd s_2 \nonumber 
\end{align}
where in $(a)$, $\widetilde{N}_y$ and $\widetilde{Z}_{\Ydraft,\beta}(x)$ denote independent copies of $N_y$ and $Z_{\Ydraft,\beta}(x)$ conditioned on $N_0$. This can be further rewritten as,
\begin{align}
\Pr (N_0 \ge 1) \cdot \mathbb{E}_N [ (\pipartapx{N} (y|x))^2 ] &= ( \phi_\beta (y|x))^2 \int_0^\infty \int_0^\infty \Phi (s_1, s_2) \,\dd s_1 \, \dd s_2  \label{eq:Psi}
\end{align}
where we define,
\begin{align*}
\Psi_y (s_1 , s_2) = \mathbb{E} \left[ N_y \widetilde{N}_y \cdot e^{- s_1 \, Z_{\Ydraft,\beta}(x) - s_2 \, \widetilde{Z}_{\Ydraft,\beta}(x)} \right]
\end{align*}
which uses the fact that $N_y = \widetilde{N}_y = 0$ if $N_0 = 0$. The remainder of this proof will be to analyze $\Psi_y (s_1 , s_2)$. First we describe a unified way to treat $\{ N_y \}$ and $\{ \widetilde{N}_y \}$. Across $N_0 \sim \textsf{Poi} (n)|_{>0}$ trials, suppose we draw a $z^i = y_1^i y_2^i \sim \pigen (\cdot|x) \times \pigen (\cdot|x)$ in each round. At the end of these rounds, we will let $\sum_{i=1}^{N_0} \mathbb{I} (z^i_1 = y) = N_y$ and $\sum_{i=1}^{N_0} \mathbb{I} (z^i_2 = y) = \widetilde{N}_y$. For any $z \in \mathcal{Y}^2$, let $N_z$ denote the number of occurrences of $z$ among the outcomes of the trials. By the Poisson structure, we have that $N_z$ is independent across different values of $z$. Furthermore, we have that $N_y = \sum_{z \in \mathcal{Y}^2 : z_1 = y} N_z$ and $\widetilde{N}_y = \sum_{z \in \mathcal{Y}^2 : z_2 = y} N_z$. This means that,
\begin{align}
\Psi_y (s_1, s_2)
&= \mathbb{E} \left[ \left( \sum\nolimits_{z : z_1 = y} N_z \right) \left( \sum\nolimits_{\widetilde{z} : \widetilde{z}_2 = y} N_{\widetilde{z}} \right) \cdot e^{- \sum_{z'} N_{z'} \left( s_1 \phi_\beta (z'_1|x) + s_2 \phi_\beta (z'_2|x) \right)} \right] \nonumber \\
&= \mathbb{E} \left[ \left( \sum\nolimits_{\substack{z : z_1 = y \\ \widetilde{z} : \widetilde{z}_2 = y \\ z \ne \widetilde{z}}} N_z N_{\widetilde{z}} + N_{yy}^2 \right) \cdot e^{- \sum_{z'} N_{z'} \alpha (z')} \right] \label{eq:Psis1s2}
\end{align}
Where $\alpha (z') = s_1 \phi_\beta (z_1'|x) + s_2 \phi_\beta (z_2'|x)$. Breaking down the first bracket into a summation across individual terms, we will first bound for $z \ne \widetilde{z}$ such that $z_1 = y$ and $\widetilde{z}_2 = y$,
\begin{align}
    \mathbb{E} \left[ N_z N_{\widetilde{z}} \cdot e^{- \sum_{z'} N_{z'} \alpha (z')} \right]
    &= \mathbb{E} \left[ N_z e^{- N_z \alpha (z)}  \right] \mathbb{E} \left[ N_{\widetilde{z}} e^{- N_{\widetilde{z}} \alpha (\widetilde{z})}  \right] \prod_{z' \not\in \{ z, \widetilde{z} \}} \mathbb{E} \left[ e^{- N_{z'} \alpha (z')} \right] \label{eq:00199191954}
\end{align}
Let $\pigen (z'|x) = \pigen (z_1' |x) \pigen (z_2' | x)$. Then, for any $\omega \in \mathbb{R}$,
\begin{align*}
\mathbb{E} \left[ e^{-\omega N_{z'} } \right] &= \sum_{j=0}^\infty e^{-\omega j} \Pr [N_{z'} = j] \\
&= \sum_{j=0}^\infty e^{-\omega j} \cdot \frac{\big( n \pigen (z' |x) \big)^j}{j!} e^{-n \pigen (z'|x)} \\
&= \exp (n \pigen (z'|x) \big(e^{-\omega } - 1\big))
\end{align*}
For any $k \ge 0$, a similar calculation results in,
\begin{align*}
\mathbb{E} \left[ e^{-\omega N_{z'}} \cdot \prod_{l=0}^k (N_{z'} - l) \right] &= \sum_{j=0}^\infty \prod_{l=0}^k (j - l) e^{-\omega j} \Pr [N_{z'} = j] \\
&= \sum_{j=k+1}^\infty e^{-\omega j} \frac{(n \pigen(z'|x))^j}{(j-k-1)!} e^{-n \pigen(z'|x)} \\
&= \big(n \pigen(z'|x) e^{-\omega} \big)^{k+1} \cdot \sum_{j=0}^\infty e^{-\omega j} \frac{(n \pigen(z'|x))^j}{j!} e^{-n \pigen(z'|x)} \\
&= \big( n \pigen(z'|x) e^{-\omega} \big)^{k+1} \cdot \exp (n \pigen (z'|x) \big(e^{-\omega } - 1\big))
\end{align*}
Combining with \cref{eq:00199191954} (with appropriately instantiated $\omega$ and $k$),
\begin{align*}
    &\mathbb{E} \left[ N_z N_{\widetilde{z}} \cdot e^{- \sum_{z'} N_{z'} \alpha (z')} \right] \\
    &= \big( n \pigen(z|x) e^{- \alpha (z)} \big) \big( n\pigen(\widetilde{z}|x) e^{- \alpha (\widetilde{z})} \big) \prod_{z'} \exp (n \pigen (z'|x) \big(e^{-\alpha (z') } - 1\big)) \\
    &= \big( n \pigen(z|x) e^{- \alpha (z)} \big) \big( n\pigen(\widetilde{z}|x) e^{- \alpha (\widetilde{z})} \big) \exp (n \mathbb{E}_{z' \sim \pigen (\cdot|x)} \big[ e^{-s_1 \phi_\beta (z_1'|x) + s_2 \phi_\beta (z_2'|x)} - 1\big]) \\
    &= \big( n \pigen(z|x) e^{- \alpha (z)} \big) \big( n\pigen(\widetilde{z}|x) e^{- \alpha (\widetilde{z})} \big) \exp (n \big[ \Phi_\beta (s_1) \Phi_\beta (s_2) - 1\big])
\end{align*}
On the other hand,
\begin{align*}
    &\mathbb{E} \left[ N_{yy}^2 \cdot e^{- \sum_{z'} N_{z'} \alpha (z')} \right] \\
    &= \mathbb{E} \left[ N_{yy}^2 e^{- N_{yy} \alpha ( yy)}  \right] \prod_{z' \ne yy} \mathbb{E} \left[ e^{- N_{z'} \alpha (z')} \right] \\
    &= \mathbb{E} \left[ \big( N_{yy} (N_{yy} - 1) + N_{yy} \big) e^{- N_{yy} \alpha ( yy)} \right] \prod_{z' \ne yy} \mathbb{E} \left[ e^{- N_{z'} \alpha (z')} \right] \\
    &= \left( \big( n \pigen(yy|x) e^{-\alpha (yy)} \big)^2 + \big( n \pigen(yy|x) e^{-\alpha (yy)} \big) \right) \cdot \exp (n \pigen (z'|x) \big(e^{- \alpha (yy) } - 1\big)) \\
    &\qquad \times \prod_{z' \ne yy} \exp (n \pigen (z'|x) \big(e^{- \alpha (z') } - 1\big)) \\
    &= \left( \big( n \pigen(yy|x) e^{-\alpha (yy)} \big)^2 + \big( n \pigen(yy|x) e^{-\alpha (yy)} \big) \right) \cdot \exp (n \big[ \Phi_\beta (s_1) \Phi_\beta (s_2) - 1\big])
\end{align*}
Combining with \cref{eq:Psis1s2},
\begin{align*}
    \Psi_y (s_1,s_2) &= \sum_{\substack{z : z_1 = y\\\widetilde{z} : \widetilde{z}_2 = y}} \big( n \pigen(z|x) e^{- \alpha (z)} \big) \big( n\pigen(\widetilde{z}|x) e^{- \alpha (\widetilde{z})} \big) \exp (n \big[ \Phi_\beta (s_1) \Phi_\beta (s_2) - 1\big]) \\
    &\qquad + \big( n \pigen(yy|x) e^{-\alpha (yy)} \big) \cdot \exp (n \big[ \Phi_\beta (s_1) \Phi_\beta (s_2) - 1\big]) \\
    &= \Phi_\beta (s_1) \Phi_\beta (s_2) \big( n \pigen(y|x) \big)^2 e^{- (s_1+s_2) \phi_\beta (y|x)} \exp (n \big[ \Phi_\beta (s_1) \Phi_\beta (s_2) - 1\big]) \\
    &\qquad + \big( n (\pigen(y|x))^2 e^{-(s_1+s_2) \phi_\beta (y|x)} \big) \cdot \exp (n \big[ \Phi_\beta (s_1) \Phi_\beta (s_2) - 1\big]) \\
    &= \big( n^2 \Phi_\beta (s_1) \Phi_\beta (s_2) + n \big) \big( \pigen(y|x) \big)^2 e^{- (s_1+s_2) \phi_\beta (y|x)} \exp (n \big[ \Phi_\beta (s_1) \Phi_\beta (s_2) - 1\big])
\end{align*}
Plugging into \cref{eq:Psi} completes the proof.
\end{proof}

\subsection{Bounding the KL-/$\chi^2$-divergence: Proof of \Cref{theorem:soft-bon}}

Observe that,
\begin{align}
    1 + \mathbb{E}_{N \sim \textsf{Poi} (n) |_{>0}} \left[ D_{\chi^2} (\pipartapx{N} \| \pi_{\beta}^\star ) \right]
    &= \sum_{y \in \mathcal{Y}} \mathbb{E}_{N \sim \textsf{Poi} (n) |_{>0}} \left[ \frac{(\pipartapx{N} (y|x))^2}{\pi_{\beta}^\star (y|x)} \right] \nonumber
\end{align}
Recall the calculation of $\mathbb{E} [ \pipartapx{N} (y|x))^2]$ in \Cref{lemma:integral}.\\ Noting that $\sum_{y \in \mathcal{Y}} \pigen (y|x) \phi_\beta (y|x) e^{-(s_1+s_2) \phi_\beta (y|x)} = (-\Phi_\beta' (s_1+s_2))$,
\begin{align}
    &\sum_{y \in \mathcal{Y}} \mathbb{E} \left[ \frac{(\pipartapx{N} (y|x))^2}{\pi_{\beta}^\star (y|x)} \right] \nonumber\\
    &= \frac{1}{1 - e^{-n}} \left( \int_0^\infty \int_0^\infty \left( n^2 \Phi_\beta (s_1) \Phi_\beta (s_2) + n \right) \big( (-\Phi_\beta' (0)) (- \Phi'_\beta (s_1+s_2)) \big) \right. \nonumber\\
    &\hspace{18em} \times \exp (n \big[ \Phi_\beta (s_1) \Phi_\beta (s_2) - 1\big])  \dd s_1 \, \dd s_2 \bigg) \label{eq:chi^2}
\end{align}
Suppose that instead of the product $(-\Phi_\beta' (0)) (- \Phi'_\beta (s_1+s_2))$, we in fact had $(-\Phi_\beta' (s_1)) (- \Phi'_\beta (s_2))$, we will show that the double integral exactly equals $1 - e^{-n}$. Indeed, observe that,
\begin{align}
    &\int_0^\infty \int_0^\infty \left( n^2 \Phi_\beta (s_1) \Phi_\beta (s_2) + n \right) \big( (-\Phi_\beta' (s_1)) (- \Phi'_\beta (s_2)) \big) \exp (n \big[ \Phi_\beta (s_1) \Phi_\beta (s_2) - 1\big])  \dd s_1 \, \dd s_2 \nonumber\\
    &= \int_0^\infty \int_0^\infty \left( n^2 \Phi_\beta (s_1) \Phi_\beta (s_2) + n \right) \exp (n \big[ \Phi_\beta (s_1) \Phi_\beta (s_2) - 1\big])  \dd \Phi_\beta (s_1) \, \dd \Phi_\beta (s_2) \nonumber\\
    &= \int_0^1 \int_0^1 \left( n^2 u_1 u_2 + n \right) \exp (n \big[ u_1 u_2 - 1\big])  \dd u_1 \, \dd u_2 \nonumber\\
    &= \int_0^1 \int_0^1 \frac{\dd^2}{\dd u_1 \dd u_2} \exp (n \big[ u_1 u_2 - 1\big])  \dd u_1 \, \dd u_2 \nonumber\\
    &\overset{(a)}{=} g_n (1,1) - g_n (1,0) - g_n (0,1) + g_n (0,0) \nonumber\\
    &= 1 - e^{-n}. \label{eq:575784848}
\end{align}
where in $(a)$, $g_n (u_1,u_2) = \exp (n \big[ u_1 u_2 - 1\big])$. In order to bound the integral in the form present in \cref{eq:chi^2}, we split it into $2$ parts; denoting $p_n = \log(n^4)/n$,
\begin{align*}
    \mathcal{R}_{0} &= \{ (s_1,s_2) \in \mathbb{R}^2 : \min \{ \Phi_\beta (s_1), \Phi_\beta (s_2) \} \ge 1 - p_n \} \\
    \mathcal{R}_{1} &= \{ (s_1,s_2) \in \mathbb{R}^2 : \min \{ \Phi_\beta (s_1) , \Phi_\beta (s_2) \} \le 1 - p_n \}.
\end{align*}
And let $\mathcal{I}_0$ and $\mathcal{I}_1$ denote the integral \cref{eq:chi^2} (which is over $\mathbb{R}^2$) over the disjoint regions $\mathcal{R}_0$ and $\mathcal{R}_1$.

\begin{lemma}[Bounding the integral $\mathcal{I}_0$] \label{lemma:I0} Assuming that $n$ is sufficiently large so that $0 \le p_n \le \frac{1}{\sqrt{2}}$,
\begin{align*}
    \mathcal{I}_{0} \le \exp \big( p_n^2 \left( \phi_{\max}/\phi_{\min} - 1 \right)^2 \big).
\end{align*}
Recall here, that $\phi_{\max} \triangleq \max_{x \in \mathcal{X},y \in \mathcal{Y}} \phi_\beta (y|x)$ and $\phi_{\min} \triangleq \min_{x \in \mathcal{X},y \in \mathcal{Y}} \phi_\beta (y|x)$. \end{lemma}
\begin{proof}
Observe that $\mathcal{I}_{0}$ corresponds to integrating around a small neighborhood of $(0,0)$, and to this end, we first bound $(-\Phi'_\beta (0))(-\Phi'_\beta (s_1+s_2))$ for $(s_1,s_2) \in \mathcal{R}_{0}$, showing that it is approximately equal to $(-\Phi_\beta' (s_1)) (-\Phi_\beta' (s_2))$ in this regime. The resulting integral can be bounded using the calculation done in \cref{eq:575784848}.\\

As $s_1$ and $s_2$ become smaller, which is the case as $n$ grows, the approximation becomes better: this essentially follows from the fact that $f(s) = \log \big(-\Phi_\beta' (s) \big)$ is a bounded, smooth convex function and behaves locally linearly. Formally,
\begin{align}
    &f(s_1 + s_2) + f(0) - f(s_1) - f(s_2) \nonumber\\
    &= \int_0^{s_1} \int_0^{s_2} f'' (u_1 + u_2) \dd u_1 \dd u_2 \nonumber\\
    &\overset{(a)}{=} \int_0^{s_1} \int_0^{s_2} \frac{\Phi'''_\beta (u) \Phi'_\beta (u) - (\Phi''_\beta (u))^2}{(\Phi'_\beta (u))^2} \dd u_1 \dd u_2 \nonumber\\
    &= \int_0^{s_1} \int_0^{s_2} \frac{\Phi'''_\beta (u) \Phi'_\beta (u) - (\Phi''_\beta (u))^2}{(\Phi'_\beta (u))^2 \Phi'_\beta (u_1) \Phi'_\beta (u_2)} \dd \Phi_\beta (u_1) \dd \Phi_\beta (u_2) \nonumber\\
    &\le \max_{u_1 \in [0,s_1]} \max_{u_2 \in [0,s_2]} \frac{\Phi'''_\beta (u) \Phi'_\beta (u) - (\Phi''_\beta (u))^2}{(\Phi'_\beta (u))^2 \Phi'_\beta (u_1) \Phi'_\beta (u_2)} \cdot \int_0^{s_1} \int_0^{s_2} \dd \Phi_\beta (u_1) \dd \Phi_\beta (u_2) \nonumber\\
    &\le p_n^2 \cdot \max_{u_1 \in [0,s_1]} \max_{u_2 \in [0,s_2]} \frac{\Phi'''_\beta (u) \Phi'_\beta (u) - (\Phi''_\beta (u))^2}{(\Phi'_\beta (u))^2 \Phi'_\beta (u_1) \Phi'_\beta (u_2)} \label{eq:004}
\end{align}
where in $(a)$, $u = u_1 + u_2$. In the $\mathcal{R}_{0}$ regime, letting $Y,Y' \overset{\text{i.i.d.}}{\sim} \pigen(\cdot|x)$,
\begin{align*}
    &\max_{u_1 \in [0,s_1]} \max_{u_2 \in [0,s_2]} \frac{\Phi'''_\beta (u) \Phi'_\beta (u) - (\Phi''_\beta (u))^2}{(\Phi'_\beta (u)^2 \Phi'_\beta (u_1) \Phi'_\beta (u_2)}\\
    &= \frac{\mathbb{E} \big[ \phi_\beta (Y|x) \phi_\beta (Y'|x) \cdot (\phi_\beta (Y|x) - \phi_\beta (Y'|x))^2 \cdot e^{-u (\phi_\beta (Y|x) + \phi_\beta (Y'|x))} \big] }{2 \mathbb{E} \big[ \phi_\beta (Y|x) \cdot e^{-u (\phi_\beta (Y|x)} \big]^2 \cdot \mathbb{E} \big[ \phi_\beta (Y|x) \phi_\beta (Y'|x) \cdot e^{-u (\phi_\beta (Y|x) + \phi_\beta (Y'|x))} \big] } \\
    &\overset{(a)}{\le} \frac{\mathbb{E} \big[ \phi_\beta (Y|x) \phi_\beta (Y'|x) \cdot (\phi_\beta (Y|x) - \phi_\beta (Y'|x))^2 \cdot e^{-u (\phi_\beta (Y|x) + \phi_\beta (Y'|x))} \big] }{2 \phi_{\min}^2 (1 - p_n)^2 \cdot \mathbb{E} \big[ \phi_\beta (Y|x) \phi_\beta (Y'|x) \cdot e^{-u (\phi_\beta (Y|x) + \phi_\beta (Y'|x))} \big]} \\
    &\le \frac{(\phi_{\max} - \phi_{\min} )^2}{2 \phi_{\min}^2 (1 - p_n)^2} \\
    &\le (\phi_{\max} / \phi_{\min} - 1)^2,
\end{align*}
where $(a)$ uses the assumption that $\Phi_\beta (u_1), \Phi_\beta (u_2) \ge 1 - p_n$ in this regime, and the last inequality assumes that $n$ is sufficiently large that $p_n \le \frac{1}{\sqrt{2}}$. Combining with \cref{eq:004},
\begin{align*}
    \frac{(-\Phi'_\beta (s_1+s_2))(-\Phi'_\beta (0))}{(-\Phi'_\beta (s_1))(-\Phi'_\beta (s_2))} \le \exp \left( p_n^2 \left( \frac{\phi_{\max}}{\phi_{\min}} - 1 \right)^2 \right).
\end{align*}
Combining with \cref{eq:575784848} completes the proof.
\end{proof}

\begin{lemma}[Bounding the integral $\mathcal{I}_1$] \label{lemma:I1} We have that,
\begin{align*}
    \mathcal{I}_{1} \le \frac{4}{n^2} \cdot \frac{\phi_{\max}}{ \phi_{\min}}
\end{align*}
\end{lemma}
\begin{proof}
In the $\mathcal{R}_1$ regime, we have that $n(1 - \Phi_\beta (s_1) \Phi_\beta (s_2)) \ge \log(n^4)$. Therefore,
\begin{align*}
    \mathcal{I}_{1} &= \frac{1}{1-e^{-n}} \left( \int_{\mathcal{R}_1} \left( n^2 \Phi_\beta (s_1) \Phi_\beta (s_2) + n \right) \big( (-\Phi_\beta' (0)) (- \Phi'_\beta (s_1+s_2)) \big) \right. \\
    &\hspace{18em} \exp (n \big[ \Phi_\beta (s_1) \Phi_\beta (s_2) - 1\big])  \dd s_1 \, \dd s_2 \bigg) \nonumber\\ 
    &\le 4 n^2 \int_{\mathcal{R}_1}  \big( (-\Phi_\beta' (0)) (- \Phi'_\beta (s_1+s_2)) \big) \exp ( - \log(n^4) ) \dd s_1 \, \dd s_2 \nonumber\\
    &= \frac{4}{n^2} \int_{\mathcal{R}_1} \mathbb{E} [ \phi_\beta (Y|x)] \cdot \mathbb{E} \big[\phi_\beta (Y|x) e^{- (s_1+s_2) \phi_\beta (Y|x)} \big] \dd s_1 \, \dd s_2 \nonumber\\
    &= \frac{4}{n^2} \cdot \mathbb{E} [ \phi_\beta (Y|x)] \cdot \mathbb{E} \big[\phi_\beta (Y|x))^{-1}\big]\\
    &= \frac{4}{n^2} \cdot \frac{\phi_{\max}}{\phi_{\min}} \\
    &\le \frac{4}{n^2} \cdot \frac{\phi_{\max}}{\phi_{\min}}
\end{align*}
\end{proof}

\subsection{Proof of \Cref{theorem:soft-bon}}

\begin{proof}
This is a consequence of the fact that by \cref{eq:chi^2},
\begin{align*}
    1 + D_{\chi^2} (\barpipartapx{n} \| \pi_{\beta}^\star ) = \mathcal{I}_0 + \mathcal{I}_1
\end{align*}
Combining with \Cref{lemma:I0,lemma:I1}, this results in the bound,
\begin{align*}
    1 + \bbE [ D_{\chi^2} (\pipartapx{N} \| \pi_{\beta}^\star ) ] &\le \frac{\phi_{\max}}{\phi_{\min}} \cdot \frac{4}{n^2} + \exp \left( p_n^2 \left( \frac{\phi_{\max}}{\phi_{\min}} - 1 \right)^2 \right).
\end{align*}
Simplifying this by noting that $z \le e^z - 1$, gives the inequality,
\begin{align*}
    1 + \bbE [ D_{\chi^2} (\pipartapx{N} \| \pi_{\beta}^\star ) ] &\le 1 + 2 \left( \exp \left( 4 p_n^2 \left( \frac{\phi_{\max}}{\phi_{\min}} \right)^2 \right) - 1 \right). 
\end{align*}
Taking a logarithm on both sides, noting that $\kl (p\|q) \le \log (1 + D_{\chi^2} (p\| q))$ and simplifying the RHS with the inequality, $\log (1 + c(e^x-1)) \le cx$ for $c \ge 1$ completes the proof.
\end{proof}

\subsection{Proof of \Cref{theorem:lb-softbon}}

Consider the following problem instance where there is a single prompt (which we will avoid stating in the reward and policy notations), and the response space $\mathcal{Y} = \{ y_0,y_1\}$. $r(y_0) = 0$ and $r(y_1) = R$, and $1-\pitarget (y_0) = \pitarget (y_1) = \frac{1}{1+\theta}$ where we denote $\theta = e^{\beta R}$ for the sake of brevity. The optimal aligned model is,
\begin{align*}
    \pi^\star_{w,\beta} (y_0) = \pi^\star_{w,\beta} (y_1) = \frac{1}{2}
\end{align*}
For this model, the Laplace transform of the score function is,
\begin{align} \label{eq:lap}
    \Phi_\beta (s) = \mathbb{E}_{Y \sim \pitarget(\cdot|x)} [e^{-s \phi_\beta (Y|x)}] = \frac{\theta e^{-s} + e^{-s \theta}}{1 + \theta}
\end{align}
From \Cref{lemma:integral},
\begin{align*}
    \barpipartapx{n} (y_0) &= \frac{\theta}{1+\theta} \cdot \frac{n}{(1 - e^{-n})} \int_0^\infty e^{n(\Phi_\beta (s)-1)} \cdot e^{-s} \dd s \\
    &= \frac{\theta}{1+\theta} \cdot \frac{n}{(1 - e^{-n})} \int_0^1 e^{n \left(\frac{\theta t + t^{\theta}}{1 + \theta} - 1\right)} \dd t
\end{align*}
where the last equation uses the structure of the Laplace transform of the score function in \cref{eq:lap}. Let $h = g (t) = \frac{\theta t + t^{\theta}}{1 + \theta}$, and thereby, $t = g^{-1} (h)$ and $\dd t = 1/g' (g^{-1} (h)) \dd h$.
\begin{align*}
    \barpipartapx{n} (y_0) &= \frac{\theta}{1+\theta} \cdot \frac{n}{(1 - e^{-n})} \int_0^1 \frac{e^{n \left(h - 1\right)}}{g' (g^{-1} (h))} \dd h \\
    &\overset{(b)}{=} \frac{\theta}{1+\theta} \cdot \frac{1}{(1 - e^{-n})} \int_0^n \frac{e^{-h_1}}{g' (g^{-1} (1 - h_1/n))} \dd h_1
\end{align*}
where in $(b)$ we plug in $h_1 = n (1-h)$. Observe that $g'(t) = \frac{\theta + \theta t^{\theta-1}}{1 + \theta} \le \frac{2 \theta}{1 + \theta}$ is an increasing function in $t$, and thereby,
\begin{align*}
    \barpipartapx{n} (y_0) &\ge \frac{1}{2}.
\end{align*}
In particular, we will show that $\left| \barpipartapx{n} (y_0) - \frac{1}{2} \right| \ge \mathcal{O}(\frac{\theta}{n})$. Since $g'$ and $g^{-1}$ are increasing functions,
\begin{align}
    \barpipartapx{n} (y_0)
    &\ge \frac{\theta}{1+\theta} \cdot \frac{1}{(1 - e^{-n})} \left[ \int_0^1 \frac{e^{-h_1}}{g' (g^{-1} (1))} \dd h_1 + \int_1^{n} \frac{e^{-h_1}}{g' ( g^{-1} (1 - 1/n))} \dd h_1 \right] \nonumber\\
    &= \frac{\theta}{1+\theta} \cdot \frac{1}{(1 - e^{-n})} \left[ (1 - e^{-1}) \frac{1 + \theta}{2 \theta} + \frac{e^{-1} - e^{-n}}{2\theta/(1+\theta)} \cdot \frac{2\theta/(1+\theta)}{g' ( g^{-1} (1 - 1/n))} \right] \nonumber\\
    &= \frac{1}{2} + \frac{e^{-1} - e^{-n}}{2} \left( \frac{2\theta/(1+\theta)}{g' ( g^{-1} (1 - 1/n))} - 1 \right) \label{eq:007}
\end{align}
Observe that,
\begin{align*}
    g \left(1 - \frac{1+\theta}{2\theta n} \right) &= \frac{\theta}{1+\theta} \left( 1 - \frac{1+\theta}{2 \theta n} \right) + \frac{1}{1+\theta} \left( 1 - \frac{1+\theta}{2\theta n} \right)^{\theta} \\
    &\ge \frac{\theta}{1+\theta} \left( 1 - \frac{1+\theta}{2 \theta n} \right) + \frac{1}{1+\theta} \left( 1 - \frac{1+\theta}{2 n} \right) \\
    &= 1 - \frac{1}{n}
\end{align*}
Since $g$ (and thereby $g^{-1}$) is increasing, this results in the inequality, $g^{-1} (1 - 1/n) \le 1 - \frac{1+\theta}{2\theta n}$. Since $g'$ is also increasing, we have the inequality,
\begin{align*}
    g' (g^{-1} (1 - 1/n)) &\le g' \Big( 1 - \frac{1+\theta}{2\theta n} \Big) \\
    &= \frac{\theta}{1+\theta} + \frac{\theta}{1+\theta} \left( 1 - \frac{1+\theta}{2\theta n} \right)^{\theta-1} \\
    &\le \frac{2\theta}{1+\theta} - \frac{\theta}{1+\theta} \left[ 1 - \left( 1 - \frac{1+\theta}{2\theta n} \right)^{\theta-1} \right]
\end{align*}
Assuming that $n \ge \theta \ge 2$, $\left( 1 - \frac{1+\theta}{2\theta n} \right)^{\theta-1} \ge 1 - \frac{1+\theta}{4n}$. This implies,
\begin{align*}
    g' (g^{-1} (1 - 1/n)) &\le \frac{2\theta}{1+\theta} - \frac{\theta}{4n}.
\end{align*}
Combining with \cref{eq:007},
\begin{align*}
    \barpipartapx{n} (y_0) - \frac{1}{2} &\ge \frac{e^{-1} - e^{-n}}{2} \left( \frac{2\theta/(1+\theta)}{2\theta/(1+\theta) - \theta/4n} - 1 \right) \\
    &\ge \frac{e^{-1} - e^{-n}}{2} \left( \frac{2}{2 - \theta/4n} - 1 \right) \\
    &\ge \frac{\theta}{50n},
\end{align*}
where the last inequality assumes that $n \ge 3$. This implies that,
\begin{align*}
    D_{\text{TV}} (\barpipartapx{n}  ,\pi_{\beta}^\star) \ge \frac{\theta}{50n}.
\end{align*}
The proof concludes by an application of Pinsker's inequality to lower bound the KL divergence by the squared TV distance, $D_{\text{TV}}^2 (p, q) \le \frac{1}{2} \kl (p,q)$). Finally, by the convexity of the KL divergence, we may take the expectation over $N \sim \Poi (n)|_{>0}$ outside the KL divergence.

\section{Additional Experiments for \specalign}
\label{sec:further_experiments}

\subsection{Latency Breakdown}
\label{sec:latency_breakdown}
In this experiment, we further hope to break down and understand the distribution of latency cost across the various components of \specalign: generation from the draft model (when deferred to), generation from the target model, and scoring using the draft and target models. We evaluate these across all three datasets, AMC23, MATH500 and OlympiadBench when the draft-target pair is Qwen2.5-1.5B-Instruct and Qwen2.5-7B-Instruct. We observe that generation from the target model is still a large component of the latency breakdown, since it is slower than the other steps. However, the overall latency can be brought down via generation from the draft model.

\begin{figure}[h]
    \centering
\includegraphics[width=\linewidth]{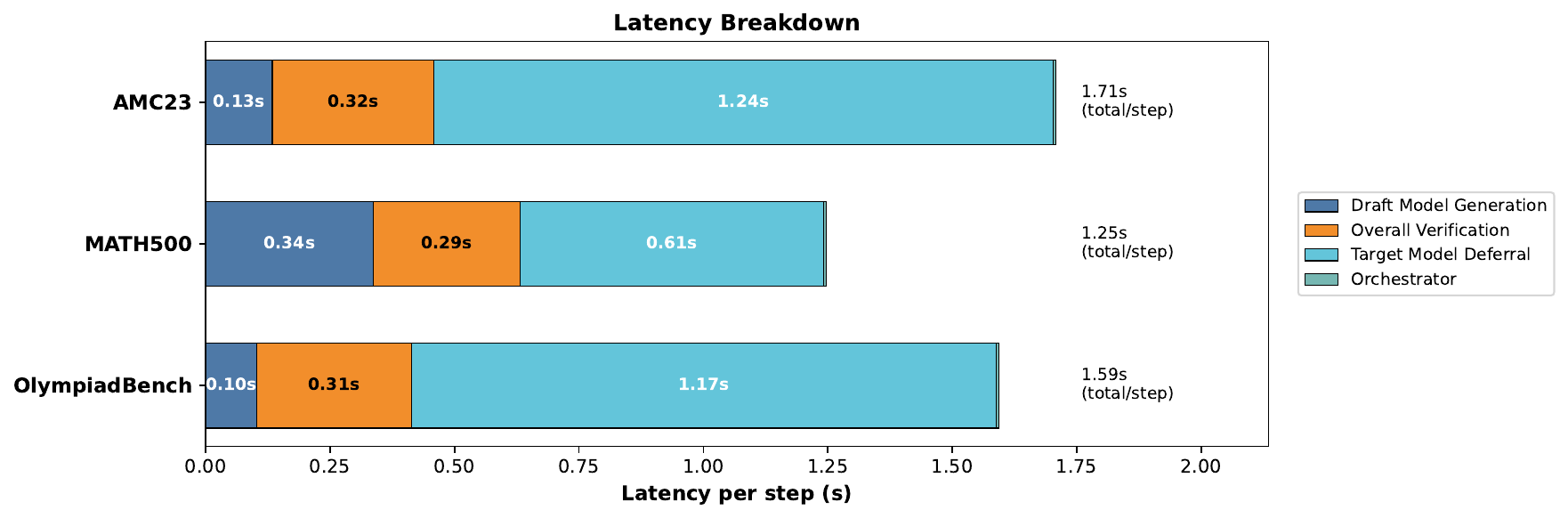}
    \caption{Latency breakdown of \specalign}
    \label{fig:latency_breakdown}
\end{figure}

\subsection{Hardware and frameworks used for the experiments}
All experiments are performed on a node equipped with 3 NVIDIA A100 Tensor Core GPUs. Each model is served on a separate GPU and accessed using vLLM~\citep{kwon2023efficient} API endpoint. For both beam search and \specalign, we vary the computational budget (e.g., by adjusting the beam width $n$) to trace out different latency-accuracy scalings.

\subsection{Engineering and Performance Optimizations:}
In order to make the \specalign algorithm more latency-efficient, we implemented some performance optimizations for increased parallelism and efficiency that are worth noting here. In particular, within the \subsample~subroutine in \Cref{alg:subsample}, the log-likelihoods of the generated beams under the target model as well as their scores under the reward model are evaluated. Both of these options can be carried out concurrently to only pay for the latency cost of the slower of these two operations. Moreover, during the generation of new candidate beams for a given step, ideally we would like to use prefix-caching so that we do not need to recompute the logits of the previously decoded steps. However, prefix-caching is disabled in vLLM when we also want to access the log probabilities of the prompt tokens, which is necessary for the \subsample~subroutine that requires the computation of log-likelihoods of trajectories under the target model. We solve this problem by modifying the vLLM source code to enable the computation of log probabilities. As a consequence, we do not need to compute the log-probabilities of the tokens that are prefix-cache hit, and only need the log-probabilities of the tokens appearing in the last step. We solve this issue by bypassing the prefix-cache hit tokens for the previous steps and efficiently computing the log-probabilities of the newly generated tokens in the last step, avoiding redundant computations.

\subsection{Hyperparameters for \specalign experiments}
The hyperparameters used in the experiments in Table \ref{tab:results_table} are presented in Table \ref{tab:hyperparam-selection}. Importantly, there are 3 key hyperparameters in this algorithm: the beam-width $n$, inverse-temperature for soft verification $\beta$, and the dynamic switching threshold $\tau$. For the MATH500 and OlympiadBench datasets, hyperparameters are chosen on a separate set of 100 questions from that dataset, treated as a validation set. For AMC23 dataset, since there were only 40 questions, we select the hyperparameters that give the best accuracy vs. latency for different choices of $n$ on the entire dataset, as is common practice for small datasets.


\begin{table}[h]
  \caption{Chosen hyperparameters $\beta$ and $\tau$ for different datasets; $\mathbf{n}$ is the beam width.}
  \label{tab:hyperparam-selection}
  \centering
  \begin{tabular}{l|l|c|cc}
    \toprule
    \textbf{Family} & \textbf{Dataset} & \(\mathbf{n}\) & \(\boldsymbol{\beta}\) & \(\boldsymbol{\tau}\) \\
    \midrule
    \multirow{9}{*}{Qwen Family}
      & \multirow{3}{*}{AMC23}         & 4  & 10000 & 0.8 \\
      &                                & 8  & 1000 & 0.8 \\
      &                                & 16 & 1000 & 0.9 \\
      \addlinespace
      & \multirow{3}{*}{MATH500}       & 4  & 1000 & 0.8 \\
      &                                & 8  & 10000 & 0.75 \\
      &                                & 16 & 1000 & 0.8 \\
      \addlinespace
      & \multirow{3}{*}{OlympiadBench} & 4  & 1000 & 0.9 \\
      &                                & 8  & 100  & 0.9 \\
      &                                & 16 & 1000 & 0.7 \\
    \addlinespace
    \midrule
    \multirow{9}{*}{Llama Family}
      & \multirow{3}{*}{AMC23}         & 4  & 1000 & 0.9 \\
      &                                & 8  & 1000 & 0.8 \\
      &                                & 16 & 1000 & 0.85 \\
      \addlinespace
      & \multirow{3}{*}{MATH500}       & 4  & 100 & 0.8 \\
      &                                & 8  & 1000 & 0.8 \\
      &                                & 16 & 100 & 0.7 \\
      \addlinespace
      & \multirow{3}{*}{OlympiadBench} & 4  & 10000 & 0.8 \\
      &                                & 8  & 100  & 0.6 \\
      &                                & 16 & 1000 & 0.7 \\
    \bottomrule
  \end{tabular}
\end{table}




\begin{figure}[h]
    \centering
    \begin{subfigure}[b]{0.42\textwidth}
    \centering
    \includegraphics[width=\linewidth]{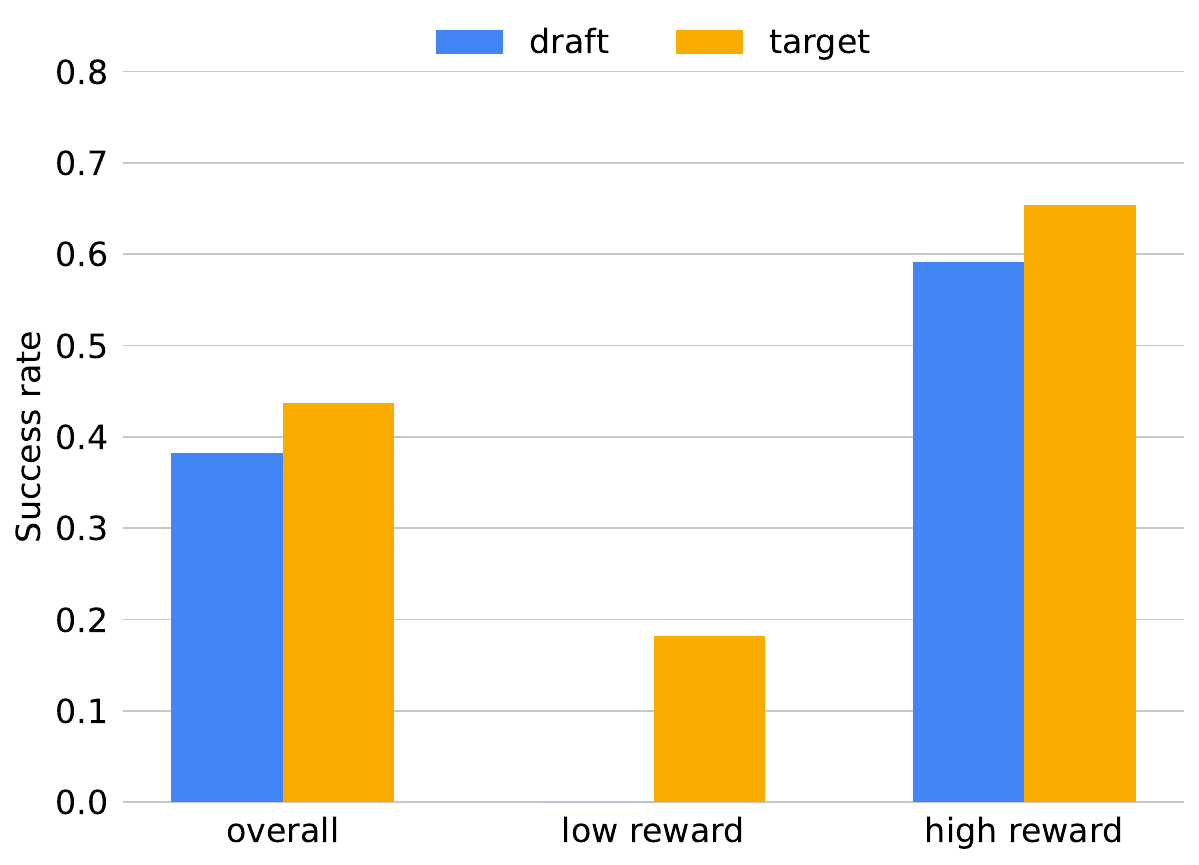}
    \caption{``High reward'' threshold is $0.6$}
    \label{fig:7a}
  \end{subfigure}%
  \qquad
    \begin{subfigure}[b]{0.42\textwidth}
    \centering
    \includegraphics[width=\linewidth]{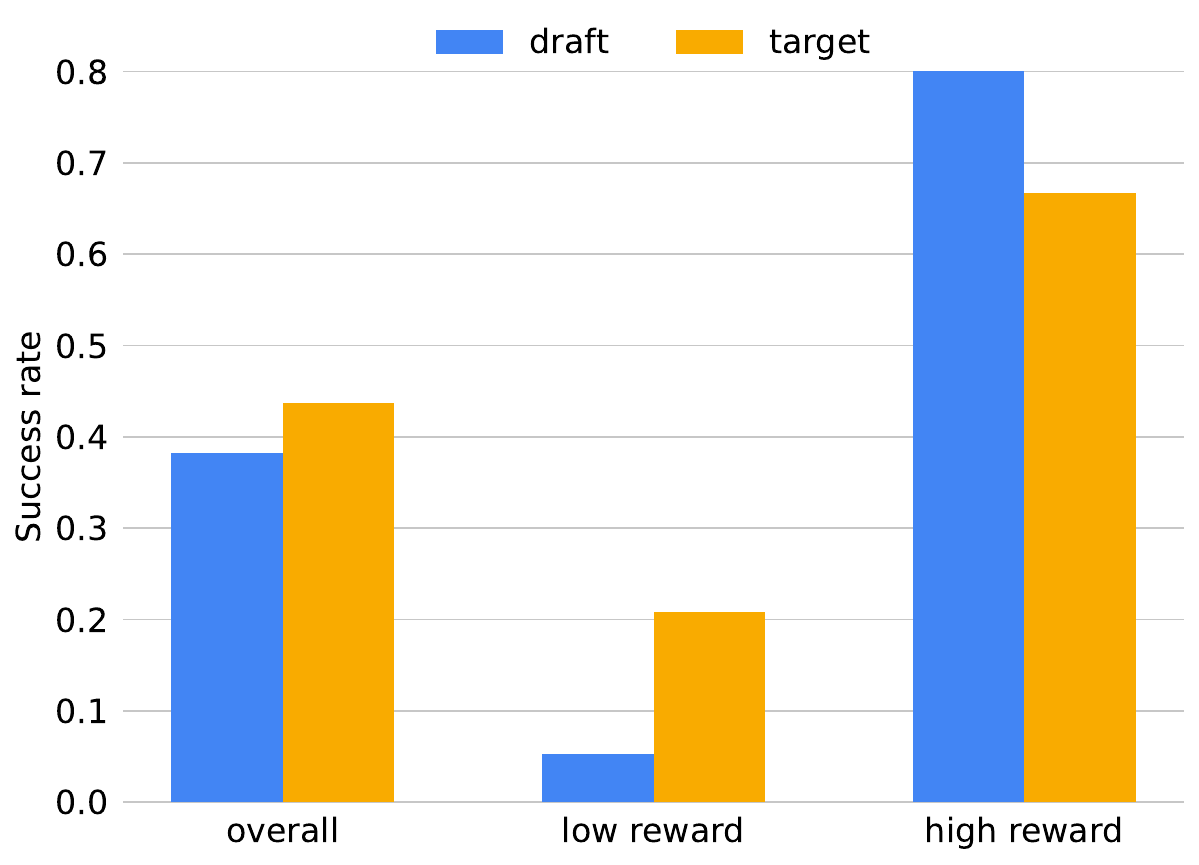}
    \caption{``High reward'' threshold is $0.7$}
    \label{fig:7b}
  \end{subfigure}
    \caption{The threshold on reward determining the batch of ``high-reward'' $8$-step partial traces generated by the target model is varied, within the same experimental setup of \Cref{fig:1b}. The trend persists very strongly: the draft model is much better at completing high-reward partial traces compared to low-reward partial traces generated by the target model.}
    \label{fig:7}
\end{figure}

\subsection{Dynamic switching}

In this section, we carry out some further analyses of how the accuracy varies as the reward threshold which defines ``high-reward'' and ``low-reward'' is varied in the setting of \Cref{fig:1b}, where it is set as $0.5$. As we increase the reward threshold defining high-reward, we observe that the trend grows stronger: the draft model gets increasingly better at solving high-reward traces. The results are described in \Cref{fig:7}.

We further try to understand the effect of switching from target model as generator to draft model as generator, as a function of the point at which traces are sliced. In \Cref{fig:1b} we observed that the draft model is very successful at completing $8$-step high-reward partial traces generated by the target model, while it is much worse at completing the corresponding low-reward partial traces. In this section, we run the same experiment but set the switchpoint to be earlier, by generating only $4$ initial steps of reasoning from the target model instead of $8$. The results of this experiment are plotted in \Cref{fig:step4_h0p5}. We observe that the draft model continues to achieve higher accuracy when completing high-reward traces compared to when completing low-reward traces. However, when we set the switchpoint to be earlier, we observe that completing with either draft or target model achieves similar performance on the low-reward traces, while there is a deterioration in performance on high-reward traces when the draft model is used to complete them.

Based on this experiment, we conclude that the trend observed in \Cref{fig:1b} is stronger when we switch from the target to the draft model at later points: completing high-reward traces using the draft model is better than using the target model, when the high-rewards are observed later in the trace.

\begin{figure}[h]
    \centering
    \includegraphics[width=0.5\linewidth]{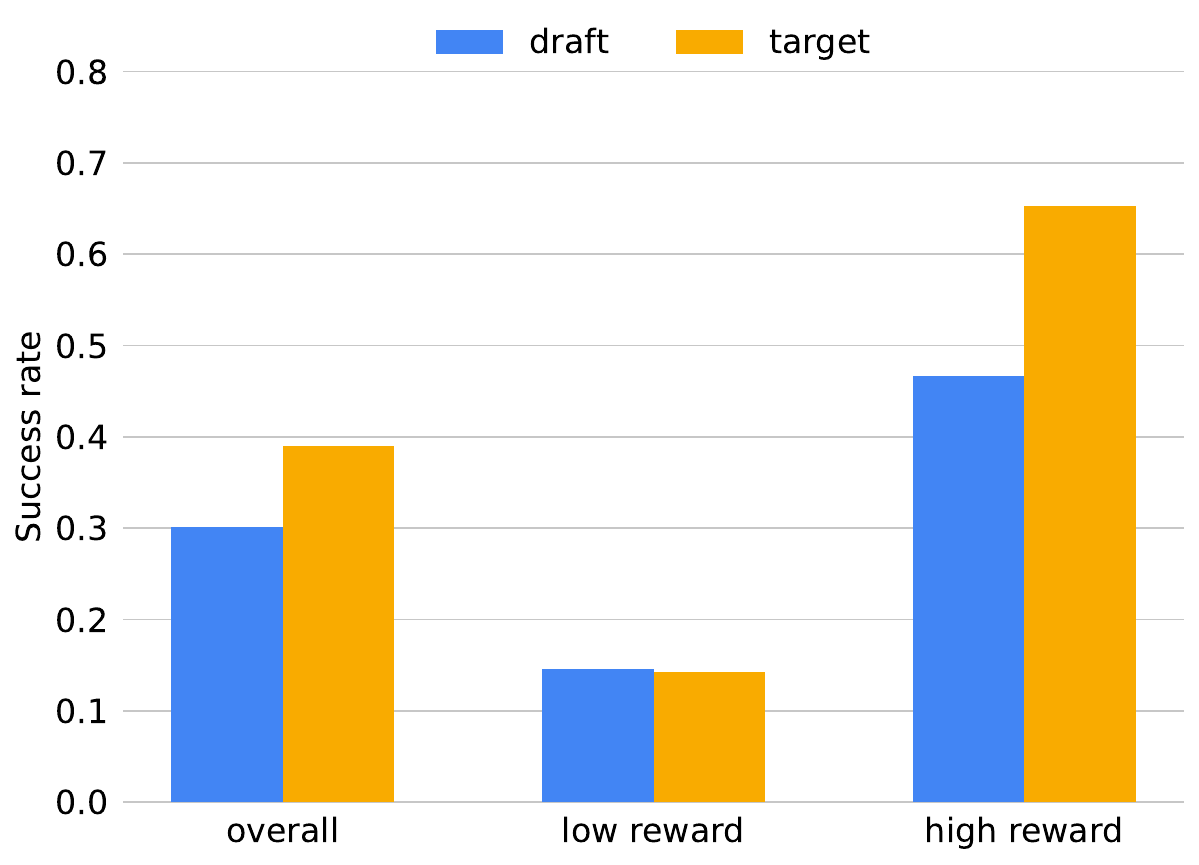}
    \caption{We generate $4$ steps of reasoning from the target model. The remaining steps are completed either by the target model, or by the draft model. We bucket the initial $4$-step reasoning traces generated by the target model into ``low-reward'' and ``high-reward'' based on the PRM reward of the trace, where the threshold is set as reward $= 0.5$. When contrasted with \Cref{fig:1b}, which is the same experiment, but with $8$ initial steps of reasoning generated by the target model, there are some differences.}
    \label{fig:step4_h0p5}
\end{figure}

{

}

\subsection{Robustness to Noisy Reward Signals}
A key concern for approaches based on reward-guided generation is the reliability of the Process Reward Model (PRM). To assess \specalign, in a controlled manner, howe \specalign tolerates imperfections in the PRM, we simulate a degraded reward signal by injecting Gaussian noise: $\tilde{r}_{\texttt{PRM}} = r_{\texttt{PRM}} + Z$, where $Z \sim \mathcal{N}(0, 0.01)$ is an independent zero-mean Gaussian with standard deviation $0.1$. As shown in Table \ref{tab:noisy_prm}, \specalign demonstrates remarkable robustness. Even with a noisy signal, the performance gap widens between \specalign and the baseline which is beam search with the small model. At the same time \specalign remains competitive with beam search with the large model. Crucially, we did not retune hyperparameters for this noisy setting, suggesting that \specalign does not require a perfect reward signal to function effectively.

{
\begin{table}[h]
\caption{Robustness of \specalign under a noisy PRM ($\tilde{r} = r + Z$ where $Z \sim \mathcal{N}(0,0.01)$). \specalign maintains high success rates compared to baselines, demonstrating that it does not rely on perfect reward signals.}
\centering
\resizebox{0.8\linewidth}{!}{
\begin{tabular}{l|c|cc|cc|cc}
\toprule
\textbf{Method} & $n$ & \multicolumn{2}{c}{\textbf{AMC23}} & \multicolumn{2}{c}{\textbf{MATH500}} & \multicolumn{2}{c}{\textbf{OlympiadBench}} \\
\cmidrule(lr){1-8}
& & \textbf{Acc} (\%) $\uparrow$ & \textbf{Lat} (s) $\downarrow$ & \textbf{Acc} (\%) $\uparrow$ & \textbf{Lat} (s) $\downarrow$ & \textbf{Acc} (\%) $\uparrow$ & \textbf{Lat} (s) $\downarrow$ \\
\cmidrule(lr){1-8}
BS(Small) w/ $\tilde{r}$ & 4 & 36.3 & 9.55  & 53.3 & 7.07  & 26.0 & 11.4 \\
BS(Large) w/ $\tilde{r}$ & 4 & 43.3 & 14.4 & 76.7 & 11.2 & 30.7 & 14.4 \\
\specalign w/ $\tilde{r}$ & 4 & \textbf{46.7} & 16.5 & \textbf{78.0} & 11.1 & \textbf{30.7} & 14.1 \\
\midrule
BS(Small) w/ $\tilde{r}$ & 8 & 34.4 & 12.1 & 62.0 & 9.77  & 25.0 & 15.8 \\
BS(Large) w/ $\tilde{r}$ & 8 & 50.8 & 20.6 & \textbf{78.7} & 14.2 & 34.0 & 20.3 \\
\specalign w/ $\tilde{r}$ & 8 & \textbf{50.8} & 19.6 & 71.3 & 12.5 & \textbf{35.3} & 20.0 \\
\midrule
BS(Small) w/ $\tilde{r}$ & 16 & 43.1 & 19.7 & 59.3 & 15.3 & 29.3 & 25.4 \\
BS(Large) w/ $\tilde{r}$ & 16 & 46.7 & 26.6 & \textbf{73.3} & 18.6 & 32.7 & 25.5 \\
\specalign w/ $\tilde{r}$ & 16 & \textbf{49.2} & 27.6 & 72.7 & 18.8 & \textbf{37.0} & 28.0 \\
\bottomrule
\end{tabular}
}
\label{tab:noisy_prm}
\end{table}
}

\end{document}